\documentclass[12pt]{article}
\usepackage{amsmath}
\usepackage{times}
\usepackage{graphicx}
\usepackage{color}
\usepackage{multirow}
\usepackage[authoryear]{natbib}
\usepackage{rotating}
\usepackage{bbm}
\usepackage{latexsym}

\usepackage{epstopdf}

\makeatletter
\@addtoreset{equation}{section}
\makeatother
\renewcommand{\theequation}{\arabic{section}.\arabic{equation}}

\newcommand{\captionfonts}{\normalsize}

\makeatletter
\long\def\@makecaption#1#2{%
  \vskip\abovecaptionskip
  \sbox\@tempboxa{{\captionfonts #1: #2}}%
  \ifdim \wd\@tempboxa >\hsize
    {\captionfonts #1: #2\par}
  \else
    \hbox to\hsize{\hfil\box\@tempboxa\hfil}%
  \fi
  \vskip\belowcaptionskip}
\makeatother

\usepackage{graphicx}

\usepackage{amsmath}
\usepackage{mathrsfs}

\usepackage{amsthm}

\usepackage{pmat}
\usepackage{subfig}
\usepackage{caption}

\usepackage{amssymb}

\usepackage{enumitem}

\usepackage[retainorgcmds]{IEEEtrantools}

\renewcommand{\theequationdis}{{\normalfont (\theequation)}}
\renewcommand{\theIEEEsubequationdis}{{\normalfont (\theIEEEsubequation)}}

\newtheorem{thm}{Theorem}
\newtheorem{lem}{Lemma}
\newtheorem{dfn}{Definition}
\newtheorem{prp}{Proposition}
\newtheorem{assm}{Assumption}
\newtheorem*{rmk}{Remark}
\newtheorem{rmk-2}{Remark}
\newtheorem{rmk-3}{Remark}
\newtheorem{rmk-4}{Remark}
\newtheorem{rmk-5}{Remark}
\newtheorem{rmk-6}{Remark}
\newtheorem{rmk-7}{Remark}
\newtheorem{rmk-8}{Remark}
\newtheorem{cl}{Corollary}

\usepackage[hyperindex,breaklinks]{hyperref}
\hypersetup{
           breaklinks=true,   
           colorlinks=true,   
           pdfusetitle=true,  
        }

\begin{document}
\hspace{13.9cm}

\ \vspace{20mm}\\

{\LARGE \flushleft On a Mechanism Framework of Autoencoders}

\ \\
{\bf \large Changcun Huang}\\
{cchuang@mail.ustc.edu.cn}\\
%


\thispagestyle{empty}
\markboth{}{NC instructions}
\ \vspace{-0mm}\\
%
\begin{center} {\bf Abstract} \end{center}
This paper proposes a theoretical framework on the mechanism of autoencoders. To the encoder part, under the main use of dimensionality reduction, we investigate its two fundamental properties: bijective maps and data disentangling. The general construction methods of an encoder that satisfies either or both of the above two properties are given. The generalization mechanism of autoencoders is modeled. Based on the theoretical framework above, we explain some experimental results of variational autoencoders, denoising autoencoders, and linear-unit autoencoders, with emphasis on the interpretation of the lower-dimensional representation of data via encoders; and the mechanism of image restoration through autoencoders is natural to be understood by those explanations. Compared to PCA and decision trees, the advantages of (generalized) autoencoders on dimensionality reduction and classification are demonstrated, respectively. Convolutional neural networks and randomly weighted neural networks are also interpreted by this framework.


\section{Introduction}
The curse of dimensionality is that the number of parameters or data points required for the learning task grows exponentially with respect to the dimensionality of the data. An autoencoder \citep*{Hinton2006} has a close relationship with this issue, because its encoder part could reduce the dimensionality and give the higher-dimensional data a lower-dimensional representation.

A convolutional neural network \citep*{LeCun1989,LeCun1998,Krizhevsky2017} can be regarded as a type of autoencoder. As will be proved in this paper, its subnetwork except for the last several fully connected layers is equivalent to an encoder, and the dimensionality of its output is usually much less than that of the input space. The architecture of convolutional neural networks has nowadays become one of the main streams of deep learning \citep*{LeCun2015}.

Since an autoencoder has been embedded in universal network architectures in terms of unsupervised pretraining of parameters \citep*{Bengio2006, Bengio2009}, its mechanism is also related to the interpretation of general deep neural networks. Thus, the principle of autoencoders is one of the key points of the theory of deep learning.

This paper will provide a theoretical framework of autoencoders, with a focus on the lower-dimensional representation of data through encoders. As applications of our theories, we will explain some experimental results of several common autoencoders applied in practice.

\subsection{Research Methodology}
Due to the complex behavior of autoencoders, early studies mainly concentrated on experimental methods. The scholars provided many intuitive insights on the underlying principle of experimental results, such as \citet*{Bengio2006}, \citet*{Rifai2011}, and \citet*{ Vincent2010}, contributing to various successful applications, as well as a rich source of ideas for further theoretical investigations.

However, the experimental analyses have its intrinsic shortcomings, whose ambiguous description of the concept or thought may not be precise enough to model the problem. And only under the methodology of deductive systems (such as Euclidean geometry) through logical reasoning, could various conclusions be combined to yield new ones in a more accurate and structural way, such that the knowledge database can be built as large as possible, without losing precision simultaneously.

As \citet*{Huang2022}, this paper will continue to use the deductive way to develop theories. There had been some theoretical studies of autoencoders, such as \citet*{Baldi1989} and \citet*{Bourlard1988}; however, they are mainly for the autoencoders in the early days, and are not much related to the contemporary cases.

\subsection{Arrangements and Contributions}
The next two sections aim to construct a bijective map via encoders. Section 2 provides some basic concepts, notations and principles that will be used throughout this paper, as well as some elementary results about bijective maps. The concept of data disentangling is introduced in definition 9, with an example given.

Section 3 proposes a general method called ``discriminating-hyperplane way'' to construct a bijective map through encoders, by which the dimensionality of the encoding space could be an arbitrary integer less than that of the input space (theorem 5). Linear-unit networks (theorem 6) and randomly weighted neural networks (theorem 7) are investigated.

Section 4 studies the disentangling property of encoders. We will present a construction method to make the input data disentangled in the output layer of an encoder (theorem 8). The comparisons with PCA (propositions 3 and 4) and decision trees (proposition 5) are given in sections 4.2 and 4.3, respectively. The relationship of autoencoders to convolutional neural networks is discussed in section 4.4 (lemma 9 and theorem 10).

Section 5 is the modeling of the generalization mechanism of autoencoders. The concepts of local (and nonlocal) generalization (definition 13), overlapping (and non-overlapping) generalization (definition 14), minor-feature space of a map (definition 16), and so on, along with some associated basic principles, are presented.

The following two sections explain some experimental results of several common autoencoders. Section 6 provides a basic principle (theorem 13) and deals with the linear-unit autoencoder (proposition 7) and denoising autoencoder (theorem 14).

Section 7 is for variational autoencoders. The key point is the interpretation of the random disturbance to the encoder (theorem 15) and the decoder (theorem 16). We will also explain the mechanism of image restoration via autoencoders (proposition 8). Section 8 is the summary of this paper by a discussion.

\subsection{Several Notes}
This paper will directly borrow some notations, definitions, and notes from \citet*{Huang2022}. The notes of \citet*{Huang2022} can be found in its section 2.4. What follows except for the last two items are all from \citet*{Huang2022}.

\begin{itemize}
\item[1.] Definitions 1 and 2: the notations of network architectures, such as $n^{(1)}m^{(1)}k^{(1)}$.
\item[2.] Note 2: the notations associated with a hyperplane derived from a unit, such as $l^0$, $l^+$ and $l_1^+l_2^+$.
\item[3.] Definitions 4 and 5: the concepts of the activation of units.
\item[4.] Note 3: the index of a layer of neural networks, with the hidden layers starting from $1$.
\item[5.] Definition 11: the concept of distinguishable data sets.
\item[6.] Note 1: the hypothesis of the rank of a matrix.
\item[7.] Note 6: the finite cardinality of data set $D$.
\item[8.] Definition 10: the concept of having the same classification effect with respect to hyperplanes.
\item[9.] The default probabilistic model is the one from the appendix of \citet*{Huang2022}.
\item[10.] Definition 16: the concepts of open and closed convex polytopes.
\item[11.] If the network of a decoder is used for reconstruction instead of classification, we assume that its output layer is composed of linear units.
\item[12.] The units of the hidden layers of a neural network are assumed be a ReLU type \citep*{Nair2010,Glorot2011}, unless otherwise stated in sections 5.3, 6 and 7.

\end{itemize}

\section{Encoder for Bijective Maps}
\begin{dfn}
The typical network architecture of an autoencoder was illustrated in Figure 1 of \citet{Hinton2006}, whose general form can be expressed as
\begin{equation}
\mathcal{A} := m^{(1)}\prod_{i=1}^{d}n_{i}^{(1)}n_e^{(1)}\prod_{j=1}^{d'}M_{j}^{(1)}m^{(1)},
\end{equation}
where $m > n_{1}$ and $n_e < M_1$, in which the encoder is
\begin{equation}
\mathcal{E} := m^{(1)}\prod_{i=1}^{d}n_{i}^{(1)}n_e^{(1)},
\end{equation}
where $n_{i+1} < n_{i}$ for $i = 1, 2, \cdots, d-1$ and $n_e < n_d$, while the decoder is
\begin{equation}
\mathcal{D} := n_e^{(1)}\prod_{j=1}^{d'}M_{j}^{(1)}m^{(1)},
\end{equation}
where $M_{j+1} > M_{j}$ for $j = 1, 2, \cdots, d'-1$. The $d+1$th layer with $n_e$ units is called the encoding layer, and the corresponding space is called the encoding space, whose dimensionality is $n_e$.
\end{dfn}

Corollary 12 of \citet*{Huang2022} pointed out that if the encoder can realize a bijective map for data set $D$ of the input space, then any element of $D$ could be reconstructed by the decoder. Thus, the bijective map is a key point of autoencoders and will be discussed in detail by sections 2 and 3.

\subsection{Geometric Knowledge}
Because some geometric concepts or conclusions of a higher-dimensional space are not as intuitive as those of two or three-dimensional space, we reformulate them here in algebraic ways.

\begin{dfn}
Let $l_1$ and $l_2$ be two lines embedded in m-dimensional space for $m \ge 2$, whose parametric equations are $\boldsymbol{x} = \boldsymbol{x}_0 + t\boldsymbol{\lambda}$ and $\boldsymbol{x} = \boldsymbol{x}_0' + t\boldsymbol{\lambda}'$, respectively. If $l_1 \cap l_2 = \emptyset$ and $\boldsymbol{\lambda} = \alpha \boldsymbol{\lambda}'$ with $\alpha$ as a constant, we say that $l_1 \parallel l_2$, which is read as ``$l_1$ is parallel to $l_2$''.
\end{dfn}

\begin{dfn}
Let $l_{m}$ be a $m-1$-dimensional hyperplane of $m$-dimensional space $\mathbb{R}^m$ for $m \ge 2$, and let $l_n \subset \mathbb{R}^m$ be an $n$-dimensional hyperplane with $n \le m-1$. If $l_n \cap l_m = \emptyset$, we say that $l_n$ and $l_m$ are parallel to each other, denoted by $l_n \parallel l_m$.
\end{dfn}

\begin{lem}
Suppose that $l_i$ for $i = 1, 2, \cdots, n$ is a $m-1$-dimensional hyperplane of $m$-dimensional space with $n \ge 2$ and $m > n$. If $l_i$'s are not parallel to each other, then their intersection $l = \bigcap_{i=1}^nl_i$ is a $m-n$-dimensional hyperplane.
\end{lem}
\begin{proof}
The proof is by proposition 5 of \citet*{Huang2022}, where the method of determining the dimensionality of the intersection of two hyperplanes was given. The intersection of a $k_1$-dimensional hyperplane $l_{k_1}$ and a $k_2$-dimensional hyperplane $l_{k_2}$ with $k_1 \ge k_2$ is $k_2 - 1$, provided that $l_{k_1} \cap l_{k_2} \ne \emptyset $ and $l_{k_2} \nsubseteq l_{k_1}$. Apply this conclusion to $l_i$'s for $n-1$ times and this lemma is proved. For example, when $m = 3$ and $n = 2$, the intersection of two planes is a line whose dimensionality is $m - n = 1$.
\end{proof}

\begin{dfn}
To point $\boldsymbol{x}_0$ and $m-1$-dimensional hyperplane $l$ whose equation is $w^T\boldsymbol{x} + b = 0$ of $m$-dimensional space $\mathbb{R}^m$, the output of $l$ with respect to $\boldsymbol{x}_0$ is defined as $z = \sigma(w^T\boldsymbol{x}_0 + b)$, where
\begin{equation*}
\sigma(s) = \max\{0, s\}
\end{equation*}
is the activation function of a ReLU; and $z' = w^T\boldsymbol{x}_0 + b$ is called the original output of $l$. To the latter case, we use $l^+$ and $l^-$ to represent the two parts of $\mathbb{R}^m$ separated by $l$, corresponding to the positive and negative original output of $l$, respectively.
\end{dfn}

\begin{lem}
Under the notations of definition 3, if $l_n \parallel l_m$, the outputs of $l_m$ with respect to all the points of $l_n$ are equal to a constant. And $l_m$ can be translated to a position such that the shifted hyperplane $l_m'$ satisfies $l_n \subset l_m'$.
\end{lem}
\begin{proof}
Write the equation of $l_m$ as $\boldsymbol{w}^T\boldsymbol{x} + b = 0$, and the parametric equation of $l_n$ as $\boldsymbol{x} = \boldsymbol{x}_0 + \sum_{i=1}^{n}t_i\boldsymbol{\lambda}_i$. Substituting $l_n$ into $l_m$ gives
\begin{equation}
\sum_{i=1}^{n}\boldsymbol{w}^T\boldsymbol{\lambda}_it_i + b_0 = 0,
\end{equation}
where $b_0 = \boldsymbol{w}^T\boldsymbol{x}_0 + b$, which is not zero because $\boldsymbol{x}_0 \notin l_m$. In equation 2.4, $\boldsymbol{w}^T\boldsymbol{\lambda}_i = 0$ must hold for all $i$; otherwise, we could find a solution for some $t_i$, which means that $l_m \cap l_n \ne \emptyset$, contradicting the condition of this lemma.

Therefore, to any point $\boldsymbol{x}' \in l_n$, the output of $l_m$ is a constant $\sigma(\boldsymbol{w}^T\boldsymbol{x}' + b) = \sigma(b_0)$. If $l_m$ is translated to $l_m'$ whose equation is $\boldsymbol{w}^T\boldsymbol{x} + b' = 0$ with $b' = -b_0$, then $\boldsymbol{w}^T\boldsymbol{x}' + b' = 0$, that is, $l_n \subset l_m'$.
\end{proof}

\subsection{Basic Principles}
\begin{thm}
To network $m^{(1)}n^{(1)}$ for $m > n$, let $l_i$ for $i = 1, 2, \cdots, n$ be the hyperplane corresponding to the $i$th unit of the first layer, with $l_i$'s not parallel to each other. Let $D \subset \prod_{i=1}^{n} l_i^+$ be a data set of the $m$-dimensional input space, and let $l = \bigcap_{i=1}^nl_i$. If $D \subset l_D$ and $l_D \parallel l$, where $l_D$ is a $k$-dimensional hyperplane with $k \le m-n$, then the elements of $D$ will be mapped to a single point by the first layer of $m^{(1)}n^{(1)}$, and vice versa.
\end{thm}
\begin{proof}
To any element $\boldsymbol{x} \in D$, the coordinate values of the mapped point $\boldsymbol{x}'$ via the first layer are the outputs of hyperplanes $l_i$'s. Because $l_D \parallel l = \bigcap_{i=1}^{n}l_i$, we have $l_D \parallel l_i$ for all $i$. Let $D'$ be the mapped set of $D$ by the first layer. Since $D \subset l_D \parallel l_i$, by lemma 2, the $i$th coordinate values of all the elements of $D'$ are the same; because $i$ is arbitrarily selected, $D'$ actually has only one distinct point.

Conversely, if the elements of $D$ become a single point after passing through the first layer, to each hyperplane $l_i$ whose equation is $\boldsymbol{w}_i^T\boldsymbol{x} + b_i = 0$ for $i = 1, 2, \cdots, n$, all of its outputs with respect to $D$ are equal to a constant; that is, to any $\boldsymbol{x}_{\nu} \in D$ for $\nu = 1, 2, \cdots, |D|$, because $D \subset \prod_{i=1}^{n} l_i^+$, the output of $l_i$ is
\begin{equation}
\boldsymbol{w}_i^T\boldsymbol{x}_{\nu} + b_i = C_i,
\end{equation}
which is a constant for all $\nu$. Equation 2.5 is equivalent to $\boldsymbol{w}^T\boldsymbol{x}_{\nu} + b_i - C_i = 0$, which is a hyperplane $l_i'$ obtained by translating $l_i$. Thus, $D \subset l_i' \parallel l_i$. Due to the arbitrary selection of $i$, we have
\begin{equation}
D \subset l' := \bigcap_{i = 1}^n l_i',
\end{equation}
where $l'$ is the intersection of $l_i'$'s. Since $l_i' \parallel l_i$, we have $l' \parallel l_i$; and by lemma 1, the dimensionality of $l'$ is $m - n$. Then $l' \parallel l = \bigcap_{i = 1}^nl_i$. Let $l_D = l'$, and the conclusion follows.
\end{proof}

\begin{dfn}
To data set $D$ of $m$-dimensional space, if there exists no $k$-dimensional hyperplane $l_k$ for $0 \le k \le m-1$ such that $D \subset l_k$, then we say that its dimensionality is $m$. Otherwise, the dimensionality of $D$ is the minimum dimensionality of a hyperplane that contains $D$.
\end{dfn}

\begin{cl}
Under the notations of theorem 1, suppose that the dimensionality of $D$ is $k_D$. If $k_D > m - n$, the elements of $D$ will not be mapped to a single point after passing through the first layer. Only when $k_D \le m - n$, it is possible that the elements of $D$ coincide into a single one.
\end{cl}
\begin{proof}
To the first part of this corollary, if the elements of $D$ were mapped to a single point, by theorem 1, there exists a hyperplane $l'$ whose dimensionality is less than or equal to $m - n$, such that $D \subset l'$, implying $k_D \le m-n$, which is a contradiction. To the second part, theorem 1 gives an example.
\end{proof}

\begin{dfn}
Let $f: D \to D_o$ with $D, D_o \subset \mathbb{R}^m$ be a multi-dimensional piecewise linear function realized by an autoencoder of equation 2.1. The function $f$ could be regarded as the composite of two functions
\begin{equation*}
f(\boldsymbol{x}) = f_d \circ f_e(\boldsymbol{x})=f_d(\boldsymbol{x}'),
\end{equation*}
where encoder function $f_e: D \to D_e$ with $D_e \subset \mathbb{R}^{n_e}$ and decoder function $f_d: D_e \to D_o$ correspond to the encoder and decoder, respectively.
\end{dfn}

\begin{lem}
Let $f_D: D \to D$ be a identity map with $D \subset \mathbb{R}^m$. To an autoencoder $\mathcal{A}$ of equation 2.1, if its encoder function $f_e$ is bijective, the function $f_D$ could be realized by $\mathcal{A}$.
\end{lem}
\begin{proof}
This lemma is simply a restatement of corollary 12 of \citet*{Huang2022} by the new terminology of definition 6.
\end{proof}

\begin{dfn}
Given data set $D$ of $m$-dimensional space, if the outputs of a $m-1$-dimensional hyperplane $l$ with respect to the elements of $D$ are different from each other, we say that $l$ discriminates $D$.
\end{dfn}

\begin{thm}
To network $m^{(1)}n^{(1)}$ and data set $D$ of the $m$-dimensional input space, among the $n$ hyperplanes $l_i$'s for $i = 1, 2, \cdots, n$ associated with the units of the first layer, if there exists at least one hyperplane, say, $l_{\nu}$ for $1 \le \nu \le n$, such that $l_{\nu}$ can discriminate $D$, then the map $f: D \to D'$ via the first layer is bijective.
\end{thm}
\begin{proof}
To any two points $\boldsymbol{x}_1'$ and $\boldsymbol{x}_2'$ of the $n$-dimensional space of the first layer, the difference of coordinate values in one dimension is sufficient to distinguish them. Since the coordinate values of the first layer are the outputs of $l_i$'s and $l_{\nu}$ can discriminate $D$, the conclusion follows.
\end{proof}

\subsection{One Constructed Solution}
We use the preceding principles to find a solution of autoencoders. By lemma 3, the key point is the bijective map of encoders. We first present two elementary results by lemma 4 of \citet{Huang2022}, from which some concepts or construction methods are proposed and will be used throughout the remaining part of this paper.

\begin{prp}
To encoder $\mathcal{E}$ of equation 2.2, let $D$ be a data set of the input space with cardinality $|D| = n_e$, which is equal to the dimensionality of the encoding space of $\mathcal{E}$. Then $\mathcal{E}$ could realize a bijective function $f_e : D \to D_e$, where $D_e$ is the mapped data set of the encoding space, provided that the input dimensionality $m > n_e + d$ with $d$ the depth of hidden layers.
\end{prp}
\begin{proof}
The proof is based on lemma 4 of \citet{Huang2022}. First, we consider each element of $D$ as a set containing only one element, and use the concept of distinguishable data sets (definition 11 of \citet*{Huang2022}) to discriminate them by the hyperplanes associated with the first layer of $\mathcal{E}$; that is, to the $i$th element $\boldsymbol{x}_i \in D$ for $i = 1, 2, \cdots, n_e$, find a distinguishable hyperplane $l_i$ for it.

Suppose that the distinguishable order of $D$ is
\begin{equation}
\boldsymbol{x}_{j_1}, \boldsymbol{x}_{j_2}, \cdots, \boldsymbol{x}_{j_{n_e}},
\end{equation}
with $1 \le j_{\nu} \le n_e$ for $\nu = 1, 2, \cdots, n_e$ not equal to each other, which correspond to the distinguishable hyperplanes $l_{j_1}, l_{j_2}, \cdots, l_{j_{n_e}}$, respectively. Let
\begin{equation}
\boldsymbol{x}' =  \begin{bmatrix} x_1, x_2, \cdots, x_{n_1} \end{bmatrix}^T
\end{equation}
be an output vector derived from the $n_1$ units of the first layer, whose entry order is in accordance with that of the distinguishable hyperplanes $l_{j_v}$'s; and the outputs of the first layer with respect to equation 2.7 are denoted by
\begin{equation}
\boldsymbol{x}'_{j_1}, \boldsymbol{x}'_{j_2}, \cdots, \boldsymbol{x}'_{j_{n_e}},
\end{equation}
which comprise the elements of the mapped data set $D'$ of $D$ by the first layer.

To $\boldsymbol{x}'_{j_2}$ of equation 2.9, it differs form $\boldsymbol{x}'_{j_1}$ in the dimension $x_2$ of equation 2.8, since the distinguishable hyperplane $l_{j_2}$ of $\boldsymbol{x}_{j_2}$ makes the coordinate values of $\boldsymbol{x}'_{j_2}$ and $\boldsymbol{x}'_{j_1}$ in that dimension nonzero and zero, respectively. Similarly, in the order of equation 2.7, to arbitrary element after $\boldsymbol{x}_{j_2}$, say, $\boldsymbol{x}_{j_{\mu}}$ for $2 \le \mu \le n_e$, its mapped $\boldsymbol{x}'_{j_{\mu}}$ is different from $\boldsymbol{x}'_{j_2}$ in dimension $x_{\mu}$ of equation 2.8, due to the property of the distinguishable hyperplane $l_{j_{\mu}}$ for $\boldsymbol{x}_{j_{\mu}}$.

In the same way, each point of equation 2.9 could be compared to other ones and the conclusion is that they are different from each other. Thus, the map of the first layer is bijective.

Note that after constructing hyperplanes $l_i$'s for $i = 1, 2, \cdots, n_e$, the map of the first layer is already bijective; and if we add $n_1'$ more units, by theorem 2, this bijective property would not be influenced. The number of the units of the first layer is thus $n_1 = n_e + n_1'$.

Each subsequent layer should be dealt with analogous to the first layer. To the architecture of encoders, the number of units should decrease monotonically as the depth of the layer grows; in order for that, we should impose a restriction as
\begin{equation}
n_1' > n_2' > \cdots > n_d' \ge 1,
\end{equation}
where $n_j'$ for $j = 1, 2, \cdots, d$ is the number of the extra units added in each layer after constructing a bijective map by the distinguishable-set method; then we have $m > n_1 > n_2 > \cdots > n_e$, which is an architecture of encoders.

Because the encoding layer has $n_e$ units, there should be at least $n_e + d$ units in the first layer to satisfy equation 2.10, which implies $m > n_e + d$.
\end{proof}

\begin{dfn}
To a multi-category data set $D$ of $m$-dimensional space, if its each category could be linearly separated from the remaining ones, we say that $D$ is a linearly separable data set or $D$ could be linearly separable; otherwise, it is linearly inseparable. Some times we use the synonymous term ``separated'' or ``classified'' to represent ``separable''.
\end{dfn}

\begin{dfn}
To a neural network $\mathcal{N}$, suppose that a multi-category data set $D$ of its input space is linearly inseparable. If $D$ could become linearly classified in the output layer after passing through $\mathcal{N}$, we say that $D$ is disentangled by $\mathcal{N}$.
\end{dfn}

\begin{rmk}
The term ``disentangle'' frequently appeared in literatures, such as \citet*{Bengio2009} and \citet*{Goodfellow2016}, which describes the possible classification capability of neural networks. However, it has not yet been rigorously defined. Due to its importance in the mechanism of neural networks, we define it here and give an example; more detailed discussion will be in section 4.
\end{rmk}

\begin{prp}
Under the notations of definition 9, suppose that each category of $D$ has only one element. If the dimensionality $m$ of the input space satisfies $m > n_e^2 + n_e + d$, where $n_e = |D|$ is the dimensionality of the encoding space and $d$ is the depth of the hidden layers, then we can construct an encoder $\mathcal{E}$ of equation 2.2 not only realizing a bijective map, but also disentangling $D$.
\end{prp}
\begin{proof}
The first step is to construct a network $N$ by part of the proof of proposition 1 to ensure a bijective map, with each hidden layer having $n_e$ units; and by theorem 2, adding extra units in each layer could not influence the bijective property.

Second, add $n_i'$ units in the $i$th layer of $N$ for $i = 1, 2, \cdots, d$ (i.e., except for the output layer), satisfying $n_i' > n_e^2$ and $n_1' > n_2' > \cdots > n_{d}'$. The modified network is denoted by $\mathcal{E}'$, which is an encoder.

Third, add $n_e^2$ units in the encoding layer of $\mathcal{E}'$, and the details are as follows. Let $D_e'$ be the mapped data set of $D$ by $\mathcal{E}'$. In the encoding space of $\mathcal{E}'$, to each element $\boldsymbol{x}_i' \in D_e'$ for $i = 1, 2, \cdots, n_e$, according to its distinguishable hyperplane $l_i$, add $n_e$ hyperplanes (or units) by proposition 1 of \citet*{Huang2022}, each of which has the same classification effect as $l_i$. Because $|D_e'| = n_e$ and each element must be similarly dealt with, we need to add $n_e^2$ units in the encoding layer. Until now, the final network architecture is denoted by $\mathcal{E}$, which is also an encoder by the construction process.

Let $D_e$ be the mapped data set of $D$ via $\mathcal{E}$. Then if we add a unit $u_i$ for $i = 1, 2, \cdots, n_e$ in a new $d+2$th layer of $\mathcal{E}$, by proposition 3 of \citet*{Huang2022}, it could realize an arbitrary piecewise linear function on $D_e = \bigcup_{i=1}^{n_e} D_i$, where each subdomain or category $D_i$ contains only one distinct element. By corollary 3 of \citet*{Huang2022}, $u_i$ could classify $D_i$ from the remaining categories by a specially designed piecewise linear function. Because the arbitrary selection of $i$, each category of $D_e$ could be linearly separated from other ones by the corresponding unit of the $d+2$th layer. That is, $D_e$ is linearly separable and $D$ is disentangled by encoder $\mathcal{E}$.

To the architecture of $\mathcal{E}$, the number of the units of the encoding layer is $n_e^2 + n_e$ as discussed above. To fulfil the characteristic of encoder architectures, the first layer should have at least $n_e^2 + n_e + d$ units. Thus, the dimensionality of the input space must satisfy $m > n_e^2 + n_e + d$.
\end{proof}

\begin{rmk}
In this example, the zero output of a ReLU plays a central role in data disentangling. In a later section, we will further explore this property to deal with more general cases.
\end{rmk}

\section{Discriminating-Hyperplane Way}
By lemma 3, a bijective map of the encoder is a key point to the solution of an autoencoder. We have constructed such a map in proposition 1, but the dimensionality of the encoding space is restricted by the cardinality of the input data set. This section will provide another method based on theorem 2, with no constraint on the number of units of the encoding layer.

Finding a certain concrete solution is not our only purpose. By the construction process, the property of the solution space would be reflected to some extent, which is helpful to interpret other solutions; and some intermediate results are general, which may become the foundation of the theories of autoencoders.

Note that in this section, only the linear part of a ReLU is used, and the results indicate that linear units are enough to produce a bijective map.
\subsection{Preliminaries}
\begin{lem}
Let $D$ and $l$ be a data set and a $m-1$-dimensional hyperplane of $m$-dimensional space, respectively. If any line connecting arbitrary two points of $D$ is not parallel to $l$, then all the $k$-dimensional hyperplanes with $2 \le k \le m-1$ passing through at least three elements of $D$ will also not be parallel to $l$. 
\end{lem}
\begin{proof}
Denote by $P_D$ the set of the $k$-dimensional hyperplanes passing through more than two points of $D$ for all $k = 2, 3,\cdots, m-1$, and by $L_D$ the set of the lines connecting some points of $D$.  The contrapositive of this conclusion is obvious. If there exists a hyperplane $l' \in P_D$ parallel to $l$, then the lines of $l_D$ that are on $l'$ will also be parallel to $l$; and because $P_D$ contains all of the lines of $L_D$, the conclusion follows.
\end{proof}

\begin{dfn}
Using the notations of lemma 4 (including its proof), to each element of $L_D$, according to whether it is parallel to $l$ or not, we call it either a parallel direction or a unparallel direction of $l$, respectively. The set $L_D$ is called the line-direction set of $D$. Let $Y_l \subset L_D$ and $N_l = L_D - Y_1$ be the sets of the parallel and unparallel directions of $l$ with respect to $D$, respectively. If we say that $L_D$ is unparallel (or not parallel) to $l$, it means that all the elements of $L_D$ have this property; otherwise, there exists some one parallel to it.
\end{dfn}

\begin{lem}
Under definition 10, if a $m-1$-dimensional hyperplane $l$ is not parallel to the line-direction set $L_D$ of data set $D$, then a hyperplane $l'$ that can discriminate $D$ could be found by the translation of $l$.
\end{lem}
\begin{proof}
Let $\boldsymbol{w}^T\boldsymbol{x} + b = 0$ be the equation of $l$. By lemma 4, if $l$ is not parallel to $L_D$, we cannot find a hyperplane with dimensionality $2 \le k \le m - 1$ passing through more than two points of $D$ and parallel to $l$. Then if $D \subset l^+$, to the outputs of $l$ for $D$, by theorem 1, it would be impossible to produce overlapping points after $D$ passing through the unit associated with $l$.

If $D \nsubseteq l^+$, we translate $l$ as far as possible until $D \subset l'^+$ by increasing or decreasing the parameter $b$ of the equation of $l$, where $l'$ is the translated version of $l$ that we need.
\end{proof}

\begin{thm}
To any data set $D$ of $m$-dimensional space with finite cardinality $\mu = |D| \ge 2$, there exists a $m-1$-dimensional hyperplane that is not parallel to the line-direction set $L_D$ of $D$.
\end{thm}
\begin{proof}
The proof is constructive. We first arbitrarily select a $m-1$-dimensional hyperplane $l$ such that $D \subset l^+$, whose equation is
\begin{equation}
\boldsymbol{w}^T\boldsymbol{x} + b = 0.
\end{equation}
Then check all the line directions of $L_D$ with at most $\binom{\mu}{2}$ possibilities to see whether or not $l$ is parallel to $L_{D}$, whose operation will be called \textsl{line-direction check}. If $l_D$ is unparallel to $l$, the construction stops. Otherwise, further steps are required as follows.

Randomly perturb the vector $\boldsymbol{w}$ of equation 3.1, after which $l$ is changed to be $l'$. We expect that $l'$ is not parallel to $l_D$. However, there are two uncertainties that may cause the failure of the perturbation. One is that some parallel directions of $l$ may still be parallel to $l'$. The other is that some unparallel directions of $l$ may become the parallel ones of $l'$.

Therefore, the solution is not evident. An inductive method is developed to construct $l$ from an one-dimensional line. The above two uncertainties are the main concerns of the following construction, which will be referred to as \textsl{perturbation dilemma} in the remaining proof.

\textbf{Step 1}: Arbitrarily choose a line $l_{1}$ whose parametric equation is $\boldsymbol{x} = \boldsymbol{x}_0 + t\boldsymbol{\lambda}$, in which all the vectors are $m$-dimensional and the line is embedded in $m$-dimensional space $\boldsymbol{X}_m$. If $L_D$ is not parallel to $l_{1}$, let $l_{1}' : = l_{1}$ and stop the construction of a line. Otherwise, perturb $l_{1}$ to $l_{1}'$ whose equation is
\begin{equation}
\boldsymbol{x} = \boldsymbol{x}_0 + t(\boldsymbol{\lambda} + \alpha_1\boldsymbol{\varepsilon}),
\end{equation}
where $\boldsymbol{\varepsilon}$ is a $m \times 1$ random vector whose entries are uniformly distributed in $(0,1)$, and similarly for other cases of this notation in this proof.

If $|\alpha_1| \ne 0$, in one-dimensional case, the parallel directions of $l_{1}$ would not be parallel to $l_{1}'$; and if $|\alpha_1|$ is small enough, the distinction between $l_{1}'$ and $l_{1}$ would be so minor that the unparallel directions of $l_{1}$ are also not parallel to $l_{1}'$. Thus, we could adjust $\alpha_1$ to make $L_D$ unparallel to $l_1'$. For simplicity, we write equation 3.2 of $l_{1}'$ as \begin{equation}
\boldsymbol{x} = \boldsymbol{x}_0 + t\boldsymbol{\lambda}_1'
\end{equation}
with $\boldsymbol{\lambda}_1' = \boldsymbol{\lambda} + \alpha_1\boldsymbol{\varepsilon}$.

\textbf{Step 2}: Construct a two-dimensional plane $l_{2}$ embedded in $\boldsymbol{X}_m$ based on the previous $l_{1}'$ as
\begin{equation}
\boldsymbol{x} = \boldsymbol{x}_0 + t_1\boldsymbol{\lambda}_1' + t_2\boldsymbol{\lambda}_2,
\end{equation}
where $\boldsymbol{x}_0$ and $\boldsymbol{\lambda}_1'$ are both from equation 3.3, and $\boldsymbol{\lambda}_2$ is arbitrarily selected with a constraint that it is linearly independent of $\boldsymbol{\lambda}_1'$. When $t_2 = 0$, $l_2$ is reduced to $l_1'$, and thus $l_{1}' \subset l_{2}$. If $L_D$ is unparallel to $l_{2}$, the construction of a two-dimensional plane stops and let $l_{2}' := l_{2}$; otherwise, let
\begin{equation}
\boldsymbol{x} = \boldsymbol{x}_0 + t_1\boldsymbol{\lambda}_1' + t_2(\boldsymbol{\lambda}_2 + \alpha_2\boldsymbol{\varepsilon})
\end{equation}
be the equation of $l_2'$, with $\alpha_2$ a real number and $\boldsymbol{\varepsilon}$ defined in equation 3.2.

Note that if $|\alpha_2|$ is sufficiently small, vector $\boldsymbol{\lambda}_2 + \alpha_2\boldsymbol{\varepsilon}$ of equation 3.5 could also be linearly independent of $\boldsymbol{\lambda}_1'$, which ensures the two dimensions of $l_2'$. Equations 3.3, 3.4 and 3.5 imply
\begin{equation}
l_{1}' = l_{2} \cap l_{2}'
\end{equation}
when $t_2 = 0$.

We now explain why equation 3.5 is the solution that we want. Let $Y_{2}$ and $N_{2}$ be the sets of the parallel and unparallel directions of $l_{2}$ associated with $D$, respectively, with $L_D = Y_2 \cup N_2$. The \textsl{perturbation dilemma} as mentioned above could be resolved by two aspects. The first is that in equation 3.6, the intersection $l_1'$ of $l_2$ and $l_2'$ is not parallel to $L_D$, due to the construction process of step 1. This implies that the parallel directions of $l_{2}$ are not parallel to $l_{2}'$; for otherwise, there would exist a direction $\boldsymbol{\gamma} \in Y_2 \subset L_D$ parallel to both $l_{2}$ and $l_{2}'$, implying $\boldsymbol{\gamma} \parallel l_{1}'$ by equation 3.6, which is a contradiction according to the construction process of $l_{1}'$.

The second is that if $|\alpha_2|$ is small enough, the unparallel-direction set $N_2$ of $l_{2}$ is also that of $l_{2}'$, due to their too minor differences. Thus, $l_{2}'$ is not parallel to $L_D$. The final parametric equation of $l_2'$ is denoted by
\begin{equation}
\boldsymbol{x} = \boldsymbol{x}_0 + t_1\boldsymbol{\lambda}_1' + t_2\boldsymbol{\lambda}_2',
\end{equation}
where $\boldsymbol{\lambda}_2' = \boldsymbol{\lambda}_2 + \alpha_2\boldsymbol{\varepsilon}$, with $|\alpha_2|$ being zero or a sufficiently small number.

\textbf{Inductive step}: We generalize the above procedures to an inductive form. Suppose that in the $n-1$th step, an $n-1$-dimensional hyperplane $l_{n-1}'$ unparallel to $L_D$ has been constructed, with its parametric equation as
\begin{equation}
\boldsymbol{x} = \boldsymbol{x}_0 + \sum_{i = 1}^{n-1}t_i\boldsymbol{\lambda}_i',
\end{equation}
where $\boldsymbol{x}_0$ is as in equations 3.3 and 3.7. Then an $n$-dimensional hyperplane is derived from equation 3.8 as
\begin{equation}
\boldsymbol{x} = \boldsymbol{x}_0 + \sum_{i = 1}^{n-1}t_i\boldsymbol{\lambda}_i' + t_n(\boldsymbol{\lambda}_n + \alpha_n\boldsymbol{\varepsilon}),
\end{equation}
where $\boldsymbol{\lambda}_n$ is linearly independent of $\boldsymbol{\lambda}'_i$'s and $\boldsymbol{\varepsilon}$ is defined in equation 3.2, with $|\alpha_n|$ being zero or a sufficiently small number determined by whether
\begin{equation}
\boldsymbol{x} = \boldsymbol{x}_0 + \sum_{i = 1}^{n-1}t_i\boldsymbol{\lambda}_i' + t_n\boldsymbol{\lambda}_n
\end{equation}
unparallel to $L_D$ or not, respectively. The reason that equation 3.9 is the solution is similar to that of equation 3.5.

Repeat the inductive step of equations 3.8 and 3.9 until $n = m-1$, and then a final $l_{m-1}'$ of equation 3.9 is constructed, which is a $m-1$-dimensional hyperplane unparallel to the direction set $L_D$. Choose $m$ points from $l_{m-1}'$ properly by equation 3.9, and the solution form of equation 3.1 could be obtained.
\end{proof}

\begin{rmk}
According to the concept of a line-direction set $L_D$, the construction in this theorem takes each variation among the elements of $D$ into consideration, thereby preserving the data information as much as possible by only one dimension.
\end{rmk}

Corollaries 2 and 3 that follows are about how to construct a bijective map from a known hyperplane.
\begin{cl}
Given data set $D$ of $m$-dimensional space, a $m-1$-dimensional hyperplane $l'$ that is not parallel to the line-direction set $L_D$ of $D$ can be derived from an arbitrary $m-1$-dimensional hyperplane $l$, in the sense that $l'$ could approximate $l$ as precisely as possible.
\end{cl}
\begin{proof}
Denote the parametric equation of a known hyperplane $l$ by
\begin{equation}
\boldsymbol{x} = \boldsymbol{x}_l + \sum_{i=1}^{m-1}t_i\boldsymbol{\gamma}_i.
\end{equation}
By the proof of theorem 3, to construct a hyperplane unparallel to $L_D$, in the $n$th inductive step, we should first introduce a new vector $\boldsymbol{\lambda}_n$ linearly independent of $\boldsymbol{\lambda}_k$'s for $k = 1, 2, \cdots, n-1$.

The key point is that the prior information of hyperplane $l$ of equation 3.11 could be embedded in the construction of equation 3.9. Let $\boldsymbol{x}_0 = \boldsymbol{x}_l$ and $\boldsymbol{\lambda}_i = \boldsymbol{\gamma}_i$ with $1 \le i \le m-1$, where $\boldsymbol{\gamma}_i$ is from equation 3.11. Then the solution $\boldsymbol{x} = \boldsymbol{x}_0 + \sum_{i=1}^{m-1}t_i\boldsymbol{\lambda}_i'$ of $l'$ by theorem 3 could approximate equation 3.11 of $l$ arbitrarily well, through the control of the parameters $\alpha_i$'s of equation 3.9.
\end{proof}

\begin{thm}
To any data set $D$ of $m$-dimensional space, we can always find a $m-1$-dimensional hyperplane to discriminate it.
\end{thm}
\begin{proof}
Theorem 3 and lemma 5 implies this conclusion.
\end{proof}

\begin{cl}
Let $D$ be a data set of $m$-dimensional space, and let $l$ be a $m-1$-dimensional hyperplane. A hyperplane $l'$ that discriminates $D$ could be constructed on the basis of the known $l$, in the sense that $l$ can be approximated by $l'$ with arbitrary precision.
\end{cl}
\begin{proof}
The combination of corollary 2 and lemma 5 yields this result.
\end{proof}

\subsection{Reconstruction Capability}

\begin{thm}
To data set $D$ of the $m$-dimensional input space, the encoder $\mathcal{E}$ of equation 2.2 can implement a bijective map $f: D \to D'$, where $D'$ is the output of $\mathcal{E}$ with respect to $D$. And the dimensionality $n_e$ of the encoding space could be an arbitrary integer satisfying $1 \le n_e < m$.
\end{thm}
\begin{proof}
The idea is to construct a bijective map in each hidden layer of $\mathcal{E}$ by theorem 4. In the first layer, denote by $f_1: D \to D_1$ a bijective map to be realized. The unit $u_{11}$ of the first layer is constructed by the method of theorem 4, such that its corresponding hyperplane $l_{11}$ can discriminate $D$. By theorem 2, one hyperplane or unit is enough for a bijective map, so that other $n_1-1$ units of the first layer can be arbitrarily selected.

Due to the arbitrary selection of other $n_1 - 1$ units together with the specially designed $u_{11}$, the condition $D \subset \prod_{i=1}^{n_1}l_{1i}^+$ could be satisfied, which means that only the linear part of a ReLU is enough to produce a bijective map.

The subsequent layers can be similarly dealt with. Each layer implements a bijective map $f_j: D_{j-1} \to D_j$ for $j = 1, 2, \dots, d+1$ with $D_0 = D$ and $D_{d+1} = D'$, and the composite of $f_j$'s is the bijective function $f : D \to D'$ of this theorem.
\end{proof}

\begin{rmk-3}
In combination with lemma 3, an autoencoder of equation 2.1 can reconstruct any discrete data set of the input space.
\end{rmk-3}

\begin{rmk-3}
Although this construction of a bijective-map only involves the linear part of a ReLU, the map $f: D \to D'$ is nonlinear, due to dimensionality reduction between layers.
\end{rmk-3}

From the proof of theorem 5, we immediately have:
\begin{cl}
A linear-unit encoder $\mathcal{E}$ could realize a bijective map $f : D \to D'$, where $D$ is the input data set and $D'$ is the corresponding output set of $\mathcal{E}$.
\end{cl}

\subsection{Disentangling Capability}
\begin{thm}
Suppose that a multi-category data set $D$ of the input space of an encoder $\mathcal{E}$ is linearly inseparable. If $\mathcal{E}$ is constructed by theorem 5 and all the units of each layer are simultaneously activated by $D$, which means that only the linear part of a ReLU is used, then $\mathcal{E}$ cannot disentangle $D$.
\end{thm}
\begin{proof}
Let $f: D \to D_e$ be the bijective map implemented by $\mathcal{E}$, then each category of $D$ has its corresponding one in $D_e$. If $D_e$ is linearly separable but $D$ not, there exists some category, say, the $\nu$th one, which can be linearly classified from other categories in $D_e$, but is linearly inseparable in $D$. Thus, in the encoding space, we can find a hyperplane $l_{\nu}$
\begin{equation}
\boldsymbol{w}^T\boldsymbol{x} + b = 0
\end{equation}
separating category $\nu$.

Let $\boldsymbol{X}_{d}$ be the space of the $d$th layer (i.e., the last hidden layer) of $\mathcal{E}$. Suppose that the encoding layer originally has $n_d$ units, whose number is equal to that of the previous $d$th layer. Then only under the linear-output region of ReLUs, the space $\boldsymbol{X}_{d}'$ of the encoding layer is an affine transform of $\boldsymbol{X}_{d}$, denoted by $\boldsymbol{X}_{d}' = \mathcal{T}(\boldsymbol{X}_{d})$. The actual $n_e$ units of the encoding layer can be regarded as deleting $n_d - n_e$ ones from the original $n_d$ units, and the encoding space $\boldsymbol{X}_e$ is an $n_e$-dimensional subspace of $\boldsymbol{X}_{d}'$.

Therefore, the $n_e-1$-dimensional hyperplane $l_{\nu}$ of equation 3.12 can be expressed as an $n_d-1$-dimensional hyperplane
\begin{equation}
\boldsymbol{w}^T\boldsymbol{x} + \boldsymbol{0}^T\boldsymbol{x}_{\bot} + b = 0,
\end{equation}
of $\boldsymbol{X}_{d}'$, where the entries of $\boldsymbol{x}_{\bot}$ are from the $n_d - n_e$ dimensions of $\boldsymbol{X}_{d}'$ besides $\boldsymbol{X}_{e}$, and $\boldsymbol{0}$ is an $(n_d - n_e) \times 1$ zero vector. Equation 3.13 is equivalent to
\begin{equation}
\boldsymbol{w}_d^T\boldsymbol{x}_d + b = 0,
\end{equation}
where $\boldsymbol{w}_d = [\boldsymbol{w}^T, \boldsymbol{0}^T]^T$ and $\boldsymbol{x}_d = [\boldsymbol{x}^T, \boldsymbol{x}_{\bot}^T]^T$, which is denoted by $l_v'$.

Let
\begin{equation}
\mathcal{T}(D_d) = D_d',
\end{equation}
where $\mathcal{T}$ is the affine transform from $\boldsymbol{X}_{d}$ to $\boldsymbol{X}_{d}'$, and $D_d$ is the mapped data set of $D$ by the $d$th layer of $\mathcal{E}$. Then the hyperplane $l_{\nu}'$ of equation 3.14 could linearly classify the $\nu$th category of $D_d'$ from other categories. Due to the property of affine transforms, by equation 3.15, the $\nu$th category of $D_d$ can also be linearly separated.

The above procedure from the encoding layer to the $d$th layer can be similarly done from the $k$th layer to the $k-1$th layer inductively for $k = d, d-1, \cdots, 2$. Finally we would reach a conclusion that the $\nu$th category of the input $D$ could be linearly separated from other categories, contradicting the assumption at the beginning of the proof. Thus, the set $D_e$ mapped by the autoencoder $\mathcal{E}$ is also linearly inseparable and the input $D$ cannot be disentangled by $\mathcal{E}$.
\end{proof}

\begin{rmk}
By this conclusion, an encoder composed of linear units can only reduce the number of parameters and is not capable of facilitating the classification or reconstruction by data disentangling.
\end{rmk}

\subsection{Randomly Weighted Neural Networks}
A randomly weighted neural network is one of the research interests in recent years \citep*{Baek2021,Ramanujan2020,Ulyanov2018}, whose performance behavior seems to be more puzzling than the trained one. A bijective map is crucial to the mechanism of a trained neural network, and so is the randomly weighted case.

\begin{lem}
Under the probabilistic model of the appendix of \citet{Huang2022}, in $m$-dimensional space, the probability that a $m-1$-dimensional hyperplane $h$ is parallel to a given line $l$ is $0$.
\end{lem}
\begin{proof}
By lemma 2, if $l \parallel h$, hyperplane $h$ could be translated to be $h'$, such that $l \subseteq h'$. Let $\boldsymbol{w}^T\boldsymbol{x} + b = 0$ be the equation of $h$, and the translated $h'$ is
\begin{equation}
\boldsymbol{w}^T\boldsymbol{x} + b' = 0.
\end{equation}
Denote the parametric equation of line $l$ by
\begin{equation}
\boldsymbol{x} = \boldsymbol{x}_0 + t\boldsymbol{\lambda},
\end{equation}
where $\boldsymbol{\lambda}$ represents the direction of the line. Since $l \subseteq h'$, substituting equation 3.17 into equation 3.16 yields
\begin{equation*}
\boldsymbol{w}^T\boldsymbol{\lambda}t = -\boldsymbol{w}^T\boldsymbol{x}_0 - b',
\end{equation*}
whose right side equals zero due to $\boldsymbol{x}_0 \in l \subseteq h'$. Then we have
\begin{equation}
\boldsymbol{w}^T\boldsymbol{\lambda}t = 0.
\end{equation}

Because $h' \cap l = l$, equation 3.18 should have infinite number of solutions for $t$, which implies $\boldsymbol{w}^T\boldsymbol{\lambda} = 0$. Since the parallel relationship is only related to the direction of $\boldsymbol{w}$ of equation 3.16, we could normalize it to be $\boldsymbol{w}' = \boldsymbol{w} / \lVert \boldsymbol{w} \rVert$, whose length is 1. Let
\begin{equation}
I =
\begin{cases}
{\boldsymbol{w}'}^T\boldsymbol{\lambda} = 0
\\
\| \boldsymbol{w}' \| = 1
\end{cases}.
\end{equation}

To equation 3.19, from geometric viewpoints, for fixed $\boldsymbol{\lambda}$ and varied $\boldsymbol{w}'$, equation ${\boldsymbol{w}'}^T\boldsymbol{\lambda} = 0$ is a $m-1$-dimensional hyperplane passing through the origin of the coordinate system; equation $\| \boldsymbol{w}' \| = 1$ is a $m$-sphere whose center is also the origin. Thus, the $I$ of equation 3.19 is the intersection of a hyperplane and a $m$-sphere. Under the measurement of the probability space of \citet{Huang2022}, the surface area of $I$ is zero, since it lacks one dimension. Thus, given a vector $\boldsymbol{w}$, the probability of its direction satisfying equation 3.19 is 0; that is,  the probability that $l \parallel h$ is 0.
\end{proof}

\begin{thm}
Let $D$ be any data set of $m$-dimensional space, and let $l$ be an arbitrary hyperplane associated with a ReLU, with $D \subset l^+$. Then the probability that $l$ discriminates $D$ is $1$.
\end{thm}
\begin{proof}
To the line-direction set $L_D$ of $D$, the number of its elements is finite. By lemma 6, the probability that each line direction parallel to $l$ is 0. Thus, $L_D$ is not parallel to $l$ with probability 1.
\end{proof}

\begin{rmk}
This theorem indicates that it's not difficult to generate a bijective map when the parameters of a neural network are randomly set.
\end{rmk}

\section{Disentangling Mechanism}
We have given an example of encoders that can disentangle the input data and simultaneously realize a bijective map in section 2.3; however, the dimensionality of the encoding space is relevant to the cardinality of the data set. The discriminating-hyperplane way of section 3 has no this restriction, but the disentangling property cannot be assured when only the linear part of a ReLU is used.

There are three main contributions of this section. The first is a new mechanism of data disentangling by an encoder, and the restriction on the number of the units of the encoding layer will be relaxed, which is determined by the categories of a data set instead of its cardinality as in section 2.3.

The next two contributions are the applications of the theories of sections 3 and 4. The second is the comparison with both principal component analysis (PCA) and decision trees. The third is the interpretation of convolutional neural networks from the perspective of autoencoders.

\subsection{Main Results}
\begin{lem}
The architecture of a decoder of equation 2.3 could disentangle any linearly inseparable multi-category data set $D$ of the input space.
\end{lem}
\begin{proof}
Lemma 6 of \citet{Huang2022} had given a solution. For instance, the network $N$ of Figure 9c of \citet*{Huang2022} is a decoder-like architecture. In the last hidden layer of $N$, a unit can be designed to be only activated by one category of $D$. Since each unit corresponds to a unique hyperplane, this means that the mapped data set is linearly separable, that is, the input data set is disentangled.
\end{proof}

The term ``region'' had been defined in combinatorial geometry, such as \citet*{Stanley2012}. However, it is a common word frequently encountered; in order to avoid confusion, we redefine it here.
\begin{dfn}
Let $H$ be a finite set of hyperplanes of $m$-dimensional space $\mathbb{R}_{m}$. The hyperplanes of $H$ can divide the space $\mathbb{R}_{m}$ into some connected parts, whose union can be denoted by $\mathcal{R} = \mathbb{R}_{m} - H$; and we call any connect component of $\mathcal{R}$ a divided region of $H$.
\end{dfn}

We borrow the terminology of polytopes from \citet*{Grunbaum2003}. The terms \textsl{$n$-polytope} and \textsl{$k$-face} denote an $n$-dimensional polytope and a $k$-dimensional face of a polytope, respectively. An \textsl{$n$-simplex} is the simplest convex $n$-polytope, with the least $n-1$-faces whose number is $n+1$.

We now investigate the condition that a multi-category data set $D$ of $m$-dimensional space can be disentangled by an encoder. By lemma 7, the mechanism of data disentangling is mainly attributed to recursive binary classification, whose number of operations influences that of the units of hidden layers. Due to the characteristic of the encoder architecture, in each hidden layer, the number of units cannot be arbitrarily large under the constraint of the dimensionality of the input space. However, in $m$-dimensional space, a convex $m$-polytope has more than $m$ faces of dimensionality $m-1$, rendering the method of lemma 7 unsuitable for encoders. Thus, a condition is proposed below.

\begin{itemize}
\item \textbf{Embedding condition}: Data set $D$ of $m$-dimensional space is embedded in an $n$-dimensional subspace whose dimensionality $n$ is sufficiently small relative to $m$, or we say that $m$ is sufficiently large relative to $n$. In this case, we could use fewer $m-1$-dimensional hyperplanes to construct an $n$-polytope, and the underlying principle is corollary 7 of \citet*{Huang2022}. For example, only $n+1$ higher-dimensional hyperplanes with dimensionality $m-1$ are required to construct an $n$-simplex embedded in $m$-dimensional space.
\end{itemize}

\begin{lem}
Let $D$ be a multi-category data set of $m$-dimensional space satisfying the embedding condition. Suppose that $D$ is linearly inseparable and each of its categories is contained in an open convex polytope that does not contain other categories. Then we can construct an encoder of equation 2.2 to both implement a bijective map on $D$ and disentangle $D$, provided that the dimensionality $m$ of the input space is sufficiently large; and the dimensionality of the encoding space is determined by the number of the open convex polytopes.
\end{lem}
\begin{proof}
The solution is based on corollary 7, lemma 6 and proposition 6 of \citet{Huang2022} as well as lemma 7 of this section. By the assumption of this lemma, let the input data set $D$ be embedded in an $n$-dimensional subspace $\boldsymbol{X}_n$ of the $m$-dimensional input space $\boldsymbol{X}_m$ with $n < m$. The proof is composed of two parts. Part 1 constructs the encoder and Part 2 proves the bijective property.

\textbf{Part 1}. First, by corollary 7 of \citet{Huang2022}, in subspace $\boldsymbol{X}_n$, the classification of $D$ via any $n-1$-dimensional hyperplane $l_n$ could be done by a $m-1$-dimensional hyperplane $l_m$ of $\boldsymbol{X}_m$ whose intersection with $\boldsymbol{X}_n$ is $l_n$, and the outputs of $l_n$ and $l_m$ with respect to any point of $\boldsymbol{X}_n$ are equal.

Second, construct a deep-layer network by lemma 6 of \citet{Huang2022} to make $D$ disentangled as in lemma 7. Only in the first layer, we should use higher-dimensional hyperplanes to classify $D$ as discussed above. From the second layer, the construction is the same as lemma 6 of \citet{Huang2022}, which is under the background of $n$-dimensional space without being embedded in a higher-dimensional space.

Let $\mathcal{N}$ be the network obtained in this step. By the construction process of lemma 6 of \citet{Huang2022}, we can see that the number of the units of the encoding layer is determined by the number of the convex polytopes containing each category.

Third, use proposition 6 of \citet{Huang2022} to modify the network $\mathcal{N}$ of the previous step to be an encoder. By that proposition, any number of redundant units could be added without influencing the realization of an affine transform; so we can add proper number of units in each layer, such that the characteristic of the architecture of an encoder is satisfied, as in proposition 1 of section 2.3. The modified network is an encoder denoted by $\mathcal{E}$.

In this step, the embedding condition for this lemma ensures that the input layer has more units than the first layer. And the condition that dimensionality $m$ is sufficiently large makes sure that although extra units are added in some layer, the number of units cannot be greater than $m$. Both of the two constraints are required for the architecture of an encoder.

\textbf{Part 2}. By the construction process, the elements of $D$ belonging to the same category will be bijectively mapped by affine transforms. And if two elements are from different categories, they will not activate the same unit of the output layer, which means that their corresponding outputs of the network are distinct. Thus, the map of encoder $\mathcal{E}$ for $D$ is bijective.

\end{proof}

\begin{thm}
To any multi-category data set $D$ of $m$-dimensional space that is linearly inseparable and satisfies the embedding condition, we can construct an encoder of equation 2.2 to realize a bijective map of $D$ and simultaneously to disentangle it, if the dimensionality $m$ of the input space is large enough.
\end{thm}
\begin{proof}
As the proof of theorem 5 of \citet*{Huang2022}, we can decompose $D$ into subsets, such that the condition of lemma 8 is fulfilled, which follows the conclusion.
\end{proof}

\begin{rmk}
By this theorem, the encoder of an autoencoder $\mathcal{A}$ not only can reduce dimensionality, but also is capable of facilitating the work of the decoder by data disentangling. The overall effect of $\mathcal{A}$ is the cooperation of its encoder and decoder.
\end{rmk}

\subsection{Comparison with PCA}
In the area of signal processing, dimensionality reduction is usually applied to what is called \textsl{signal}. A discrete signal can be considered as a set of discrete points. Thus, the term ``input data set'' of neural networks actually refers to the same thing as ``discrete signal''; for consistency, we'll use the former one in this paper.

\begin{prp}
An autoencoder of equation 2.1 can compress any data set $D$ of the $m$-dimensional input space into set $D_e$ of arbitrary $n_e$-dimensional space with $1 \le n_e < m$ by a bijective map; and its decoder can reconstruct the original $D$ from the compressed $D_e$ without loss of accuracy.
\end{prp}
\begin{proof}
The decoder had been dealt with in corollary 12 of \citet{Huang2022}, and the encoder was discussed in theorem 5.
\end{proof}

\begin{rmk}
To PCA, the effect of dimensionality reduction depends on the scatter-property of data; and usually the more the reduction, the more difficult reconstructing the original data is. However, there's no such restriction on autoencoders. As an example, Figure 2 of \citet{Hinton2006} provided some experimental results of this comparison.
\end{rmk}

\begin{prp}
To the lower-dimensional representation of multi-category data set $D$ of $m$-dimensional space, PCA cannot be as capable of disentangling $D$ as autoencoders.
\end{prp}
\begin{proof}
The PCA method can be regarded as selecting some dimensions from a changed coordinate system \citep*{Bishop2006}, according to the covariance matrix of $D$; and the relationship between the new coordinate system and the original one could be modeled by an affine transform. Thus, PCA is a special type of the dimension-selection operation of the proof of theorem 6. Combined with theorem 8, this proposition is proved.
\end{proof}

\subsection{Comparison with Decision Trees}
\begin{dfn}
To an autoencoder of equation 2.1, its decoder part could be modified to do classification instead of reconstruction. And the restrictions on the architecture of the decoder could be relaxed, such that each hidden layer could have arbitrary number of units. We call the modified version of equation 2.1 a classification autoencoder, denoted by
\begin{equation}
\mathcal{A}_c := m^{(1)}\prod_{i=1}^{d}n_{i}^{(1)}n_e^{(1)}\prod_{j=1}^{d'}M_{j}^{(1)}\mu^{(1)},
\end{equation}
which is different from equation 2.1 in that no constraints are imposed on the decoder architecture.
\end{dfn}

\begin{thm}
A classification autoencoder of equation 4.1 can classify arbitrary multi-category data set $D$ of the $m$-dimensional input space.
\end{thm}
\begin{proof}
Theorems 5 and 8 have proved that the encoder can realize a bijective map $f_e : D \to D_e$. Theorem 10 of \cite{Huang2022} gave the method of constructing any discrete piecewise linear function $f_d : D_e \to \mathbb{R}^{\mu}$ via the decoder part of equation 4.1. By corollary 4 of \citet{Huang2022}, the mapped data set $D_e$ of $D$ could be classified by a specially designed function $f_d$. And the composite function $f = f_d \circ f_e$ is the classifier needed.
\end{proof}

Decision trees could classify multi-category data set by recursive binary classification, until only one category left in a certain region. The decoder alone would have no advantages, if its solution is obtained by lemma 6 of \citet*{Huang2022}, since its underlying principle is also recursive binary classification. However, combined with the encoder, the neural-network way could demonstrate its superiority.

\begin{prp}
Compared to decision trees, a classification autoencoder $\mathcal{A}_c$ of equation 4.1 has two advantages. First, the classification is done in a lower-dimensionality space, which may greatly save the number of parameters. Second, the encoder could disentangle the input data set $D$ to facilitate the classification via the decoder.
\end{prp}
\begin{proof}
The two parts of the proof below are for the two conclusions of this proposition, respectively.

\textbf{Part 1}. In comparison with a decision tree, the extra encoder part of $\mathcal{A}_c$ increases the number of parameters, and the construction of the decoder may introduce redundant parameters. However, there are two aspects that could save the parameters of $\mathcal{A}_c$, such that the increased ones can be compensated to some extent. The first is related to binary classification via hyperplanes. The second is about the disentangling property of $\mathcal{A}_c$.

To a decision tree for $m$-dimensional space, let $n_b$ be the number of binary classification required, each of which needs $m+1$ parameters; then the total number of the parameters is
\begin{equation}
N_d = (m + 1) \times n_b.
\end{equation}
To $\mathcal{A}_c$, due to dimensionality reduction of the encoder, if the solution of the decoder is constructed by affine transforms as lemma 6 of \citet*{Huang2022}, each hyperplane for classification has only $n_e + 1$ parameters with $1 \le n_e < m$, which is less than $m+1$ of equation 4.2. On the other hand, if data set $D$ is disentangled by the encoder of $\mathcal{A}_c$, each category needs only one $n_e-1$-dimensional hyperplane to classify it, greatly reducing the number $n_b$ of equation 4.2

To the construction of the encoder of $\mathcal{A}_c$, although the basic operation is also binary classification of hyperplanes, due to dimensionality reduction layer by layer, the number of parameters of a hyperplane would decrease as the depth of layers grows, which is not the case of equation 4.2 either.

Therefore, if the dimensionality $m$ of the input space is sufficiently large and the data structure of $D$ is complicated enough, autoencoder $\mathcal{A}_c$ may be superior to a decision tree in much less parameters required.

\textbf{Part 2}. The disentangling property of $\mathcal{A}_c$ not only could save parameters as in part 1, but also is helpful to find a solution of the decoder, since a disentangled data set is much easier to be processed.
\end{proof}

\subsection{Convolutional Neural Networks}
\begin{lem}
Each of the basic operations of convolutional neural networks, including the convolution and the max or average pooling \citep*{Goodfellow2016}, can be considered as the output of a fully-connected unit (or a unit fully connected to the previous layer).
\end{lem}
\begin{proof}
Under a two-layer neural network, the convolution operation is performed within a subspace of the input space, through a unit $u_{1i}$ of the first layer that only has connections with the input units of that subspace. If we consider the unlinked case as a connection with zero weight, then the unit $u_{1i}$ is fully connected, and the convolution operation becomes an output of a common unit of feedforward neural networks.

The pooling operation is a special case of convolutions. To the max pooling, just let the weight corresponding to the maximum value (within a subspace) of the input unit be 1 and others be 0. And the case of average pooling is to set the weights within a subspace as an average filter, and other weights are forced to be zero. Thus, the pooling operation is equivalent to the output of a fully connected unit.
\end{proof}

\begin{thm}
Let $\mathcal{E}_c$ be the part of a convolutional neural network $\mathcal{C}$ that doesn't contain the fully-connected layers. Suppose that $\mathcal{E}_c$ is built by the basic operations of lemma 9. If the number of units of the hidden layers of $\mathcal{E}_c$ monotonically decreases as the depth of the layer grows, then $\mathcal{C}$ is a classification autoencoder of equation 4.1.
\end{thm}
\begin{proof}
By lemma 9, $\mathcal{E}_c$ could be regarded as a fully connected neural network with some weights specially designed, which follows the conclusion.
\end{proof}

\begin{rmk}
By this theorem, we can interpret a convolutional neural network from the viewpoint of autoencoders, and the results of this paper are beneficial to understand it.
\end{rmk}

\section{Generalization Model}
We have gained some knowledge on how an autoencoder works from the preceding sections, on the basis of which the generalization mechanism will be investigated. This section mainly provides some general definitions and principles, and their applications or examples will be given in sections 6 and 7.

\subsection{Definition of Generalization}
Note that \citet*{Bengio2009} had already made the similar definitions of local and nonlocal generalizations. Compared to them, ours are more specific and closely related to the theories of this paper.

\begin{dfn}
Suppose that network
\begin{equation}
\mathcal{N} := \mu^{(1)}\prod_{i=1}^{d}m_{i}^{(1)}1^{(1)}
\end{equation}
realizes a piecewise linear function $f': \mathbb{R}^{\mu} \to \mathbb{R}$, and the discrete version $f : D \to \mathbb{R}$ with $D \subset \mathbb{R}^{\mu}$ is derived from the sampling of $f'$. Each linear component $f_i'$ of $f'$ corresponds to a divided region $r_i$ of $\mathbb{R}^{\mu}$ for $i = 1, 2,\cdots, \nu$, where $\nu$ is the number of the linear functions. All of $r_i$'s are formed by network $\mathcal{N}$ of equation 5.1. Let $\boldsymbol{x}_0 \in D$ and $\boldsymbol{x}$ is in the neighborhood of $\boldsymbol{x}_0$. If $\boldsymbol{x}$ is in the same divided region $r_k$ as $\boldsymbol{x}_0$ for some $1 \le k \le \nu$, we say the output $f(\boldsymbol{x})$ of $\mathcal{N}$ is a local generalization of $\boldsymbol{x}_0$; otherwise, we call $f(\boldsymbol{x})$ a nonlocal generalization.
\end{dfn}

The next concept could be defined independently of definition 13, with only some notations shared with it.
\begin{dfn}
Under the notations of definition 13, denote by $\boldsymbol{x}''$ and $\boldsymbol{x}_0''$ the mapped points of $\boldsymbol{x}$ and $\boldsymbol{x}_0$ via the $d$th layer of $\mathcal{N}$, respectively. If $\boldsymbol{x}''$ and $\boldsymbol{x}_0''$ coincide or $\boldsymbol{x}'' = \boldsymbol{x}_0''$, we call the output $f(\boldsymbol{x})$ of $\mathcal{N}$ with respect to $\boldsymbol{x}$ an overlapping generalization of $\boldsymbol{x}_0$; otherwise, it is a non-overlapping generalization.
\end{dfn}

\begin{lem}
To network $\mathcal{N}$ of equation 5.1, suppose that $l_{D}$ is an one-dimensional line embedded in the $\mu$-dimensional input space, and that $\l_D \parallel \bigcap_{i = 1}^{m_1}l_i$, where $l_i$ for $i = 1, 2, \cdots, m_1$ is the hyperplane corresponding to the unit $u_{1i}$ of the first layer of $\mathcal{N}$. Given $\boldsymbol{x}_0 \in l_D$, let $\boldsymbol{x}$ be a point of the neighborhood of $\boldsymbol{x}_0$, with $\boldsymbol{x} \in l_D$. Then to $\boldsymbol{x}_0$, the output of $\mathcal{N}$ with respect to $\boldsymbol{x}$ is an overlapping generalization.
\end{lem}
\begin{proof}
By theorem 1, all the points of $l_D$ will be mapped to a single one of the space of the first layer, and it would be impossible to discriminate them in subsequent layers. Thus, it is an overlapping generalization.
\end{proof}

\begin{dfn}
Given network $\mu^{(1)}m^{(1)}$, let $\boldsymbol{x}_0$ be a point of the $\mu$-dimensional input space, and $\boldsymbol{x}$ is a point of the neighborhood of $\boldsymbol{x}_0$. By lemma 10, if the line direction connecting $\boldsymbol{x}$ and $\boldsymbol{x}_0$ is not parallel to the intersection of the hyperplanes derived from the units of the first layer, we say that $\boldsymbol{x}$ and $\boldsymbol{x}_0$ satisfy the non-overlapping condition.
\end{dfn}

\begin{thm}
Using the notations definition 13, let $\boldsymbol{x}^{(i)}$ and $\boldsymbol{x}_0^{(i)}$ for $i = 1, 2, \cdots, d$ be the mapped points of $\boldsymbol{x}$ and $\boldsymbol{x}_0$ by the $i$th layer of $\mathcal{N}$, respectively. The output of $\mathcal{N}$ with respect to $\boldsymbol{x}$ would be a non-overlapping generalization of $\boldsymbol{x}_0$, provided that $\boldsymbol{x}^{(i)}$ and $\boldsymbol{x}_0^{(i)}$ for all $i$ satisfy the non-overlapping condition. Otherwise, there exists some layer, say, the $k$th one for $1 \le k \le d$, such that $\boldsymbol{x}^{(k)}$ and $\boldsymbol{x}_0^{(k)}$ doesn't fulfil the non-overlapping condition.
\end{thm}
\begin{proof}
The proof is simply the repeated application of lemma 10, and the conclusion is described by the terminology of definition 15.
\end{proof}

\subsection{Encoder for Generalization}
Analogous to the concept of the orthogonal complement of a vector space, we introduce some notations. Let $\boldsymbol{x} = [x_1, x_2, \cdots, x_m]^T$ be a vector of $m$-dimensional space $\mathbb{R}^m$. We write $\boldsymbol{x}$ in the form
\begin{equation}
\boldsymbol{x} = \boldsymbol{x}_{m_1} \oplus \boldsymbol{x}_{m_2}
\end{equation}
to express $\boldsymbol{x} = [\boldsymbol{x}_{m_1}^T, \boldsymbol{x}_{m_2}^T]^T$, where $m_1 + m_2 = m$, $\boldsymbol{x}_{m_1}$ is a $m_1 \times 1 $ column vector of the subspace $l_{m_1}$ spanned by the first $m_1$ dimensions of $\boldsymbol{x}$, and $\boldsymbol{x}_{m_2}$ of size $m_2 \times 1$ is derived from the remaining dimensions spanning the subspace $l_{m_2}$.

The subspace $l_{m_1}$ can be regarded as a $m_1$-dimensional hyperplane, and similarly for $l_{m_2}$, with $\boldsymbol{x}_{m_1} \in l_{m_1}$ and $\boldsymbol{x}_{m_2} \in l_{m_2}$ in equation 5.2. For all $\boldsymbol{x} \in \mathbb{R}^m$, we use $\mathbb{R}^m = l_{m_1} \oplus l_{m_2}$ to express equation 5.2 in the language of sets, and define

\begin{equation}
l_{m_1} = \mathbb{R}^m \ominus l_{m_2}.
\end{equation}

\begin{dfn}
To network $m^{(1)}n^{(1)}$ for $m > n$, let $l = \bigcap_{i=1}^{n}l_i$, whose dimensionality is $\nu = m - n$, where hyperplane $l_i$ for $i = 1, 2, \cdots, n$ corresponds to the $i$th unit of the first layer, with $l_{\nu}$ not parallel to $l_{\mu}$ for $1 \le \nu, \mu \le n$ and $\nu \ne \mu$. To any point $\boldsymbol{x} \in \prod_{i=1}^nl_i^+$, we decompose $it$ into the form $\boldsymbol{x} = \boldsymbol{x}_l \oplus \boldsymbol{x}_l^{\bot}$, where $\boldsymbol{x}_l \in l'$ and $\boldsymbol{x}_l^{\bot} \in \mathbb{R}^m \ominus l'$, with $l'$ a $k$-dimensional hyperplane satisfying $l' \parallel l$ and $k \le \nu$. Let $f: \mathbb{R}^m \to \mathbb{R}^n$ be the map of $m^{(1)}n^{(1)}$. Then we call $l$ the minor-feature space of map $f$, and $\boldsymbol{x}_l$ the minor feature of $\boldsymbol{x}$ with respect to map $f$.
\end{dfn}

\begin{rmk}
For completeness, besides encoders, the case of $m \le n$ for $m^{(1)}n^{(1)}$ should also be considered. However, this case had been intensively studied in \citet*{Huang2022, Huang2020}, with the conclusion that the input information could not be lost in terms of affine transforms.
\end{rmk}

\begin{prp}
Under definition 16, if we add a $m-1$-dimensional hyperplane not parallel to any of $l_i$'s, the dimensionality $\nu$ of the minor-feature space of map $f$ will be reduced by one.
\end{prp}
\begin{proof}
See the proof of lemma 1 together with definition 16.
\end{proof}

\begin{rmk}
As the dimensionality $\nu$ of the minor-feature space becomes larger, an overlapping generalization is more likely to occur. When $\nu = 0$, it is the case of affine transforms and there would be no possibility of overlapping generalizations.
\end{rmk}

\begin{lem}
Under definition 16, the variation of $\boldsymbol{x}_0$ of the input space caused by a minor feature would be neglected by the map $f$ of the network in terms of overlapping generalization.
\end{lem}
\begin{proof}
Suppose that a variation from $\boldsymbol{x}_0$ to $\boldsymbol{x}$ is a minor feature of $\boldsymbol{x}_0$. Then the line connecting $\boldsymbol{x}$ and $\boldsymbol{x}_0$ is parallel to the minor-feature space $l$ of $f$. By lemma 10, the conclusion follows.
\end{proof}

\begin{thm}
To encoder $\mathcal{E}$ of equation 2.2, let $\boldsymbol{x}_0$ be an arbitrary point of the input space, and point $\boldsymbol{x}$ is in its neighborhood. Denote by $\boldsymbol{x}_0^{(i)}$ and $\boldsymbol{x}^{(i)}$ for $i = 1, 2, \cdots, d$ the mapped points of $\boldsymbol{x}$ and $\boldsymbol{x}_0$ via the $i$th layer of $\mathcal{E}$, respectively. The variation from $\boldsymbol{x}_0$ to $\boldsymbol{x}$ will be neglected, provided that in some layer $k$ for $1 \le k \le d$, the difference $\boldsymbol{x}^{(k)} - \boldsymbol{x}^{(k)}_0$ between $\boldsymbol{x}^{(k)}$ and $\boldsymbol{x}_0^{(k)}$ is a minor feature of $\boldsymbol{x}_0^{(k)}$.
\end{thm}
\begin{proof}
This conclusion is the combination of lemma 11, theorem 11 and definition 16.
\end{proof}

\begin{figure}[!t]
\captionsetup{justification=centering}
\centering
\subfloat[Two-dimensional case.]{\includegraphics[width=2.1in, trim = {4.5cm 2.4cm 4.5cm 3cm}, clip]{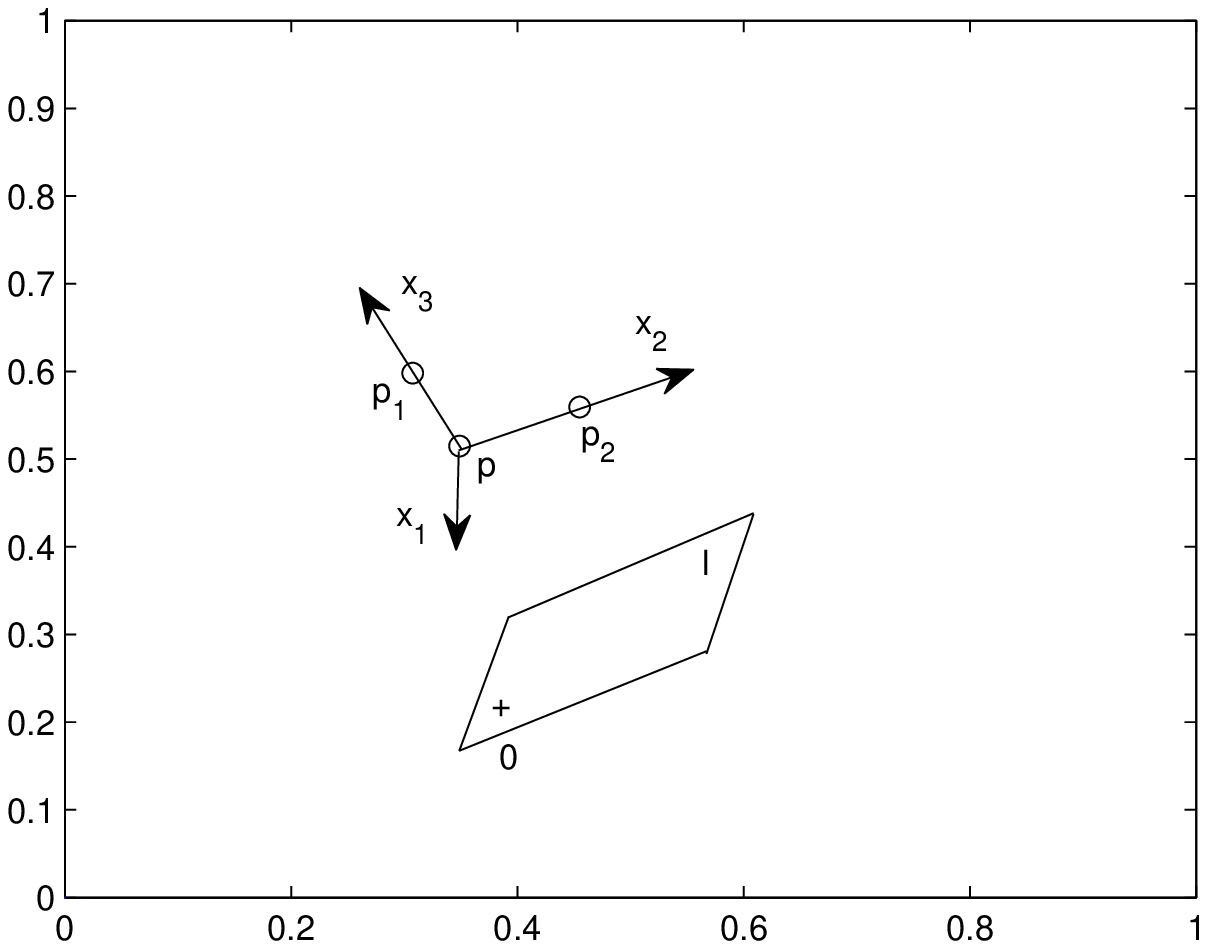}} \quad \quad \quad 
\subfloat[One-dimensional case.]{\includegraphics[width=2.1in, trim = {4.5cm 2.4cm 4.5cm 3cm}, clip]{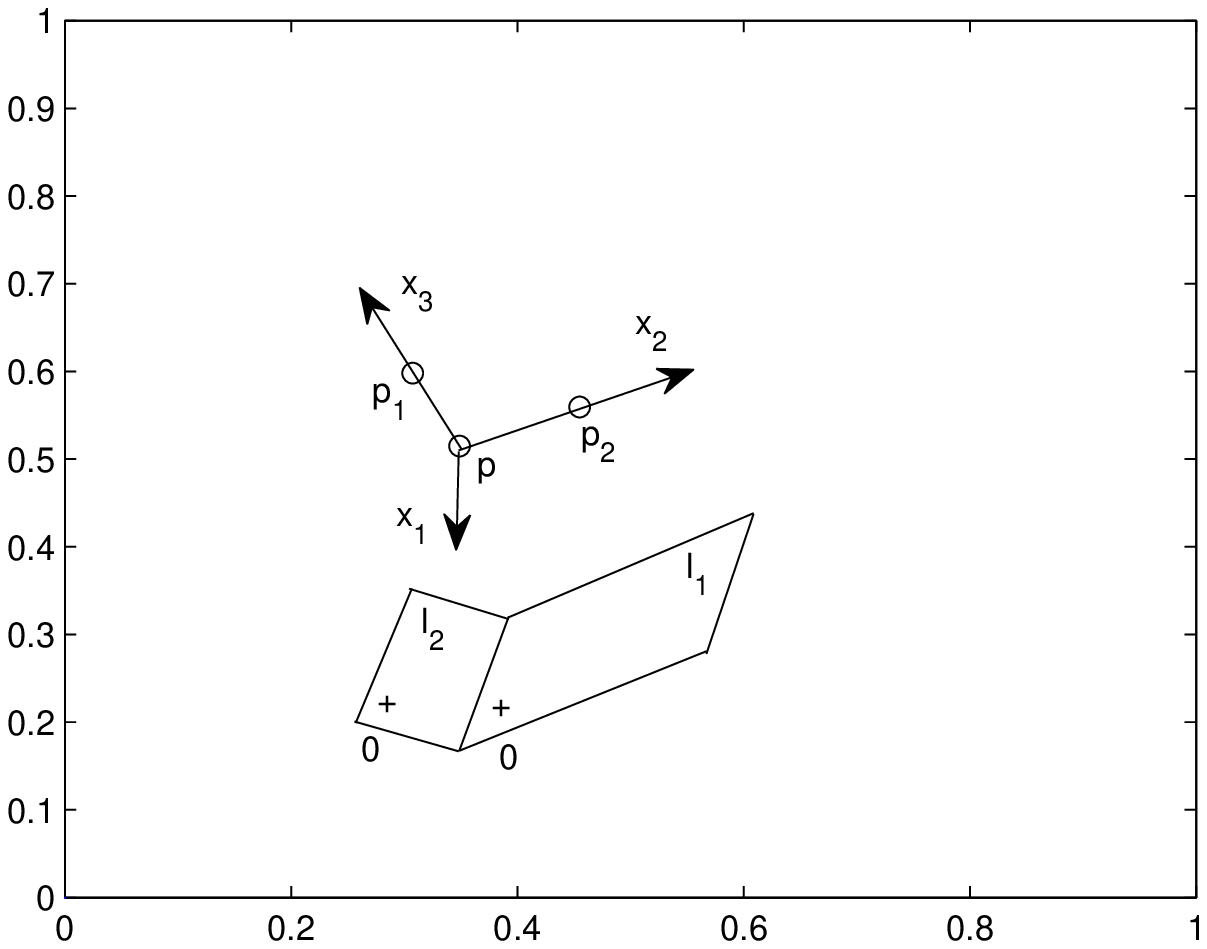}} 
\caption{Minor-feature space.}
\label{Fig.1}
\end{figure}

\noindent
\textbf{Example.} Figure \ref{Fig.1} is an example of network $3^{(1)}1^{(1)}$. As shown in Figure \ref{Fig.1}a, dimensions $x_1$, $x_2$ and $x_3$ represent a coordinate system of the three-dimensional input space; points $p_1$ and $p_2$ are in the neighborhood of $p$; plane $l$ corresponds to the only unit of the first layer of $3^{(1)}1^{(1)}$, with $l$ parallel to $x_2$ but unparallel to $x_1$ and $x_3$.

Because $x_3$ is not parallel to $l$, after passing through the first layer, the information of dimension $x_3$ could be preserved, such that the variation from $p$ to $p_1$ is reflected in the output of the network. However, to dimension $x_2$, since $x_2 \parallel l$, the difference between $p_2$ and $p_0$ would be neglected due to the reason of minor features. The minor-feature space of the map is the plane $l$ itself, whose dimensionality is 2.

If we add plane $l_2$ (with the notation of $l$ above modified to $l_1$) that is not parallel to $x_2$ as shown in Figure \ref{Fig.1}b, the information of dimension $x_2$ will be preserved by $l_2$, and the variation from $p$ to $p_2$ would not be lost as a minor feature. The minor-feature space of the modified network $3^{(1)}2^{(1)}$ is an one-dimensional line of $l_1 \cap l_2$.

\subsection{The Non-ReLU Case}
Denote the activation function of a sigmoid unit by
\begin{equation}
y = \eta(s) = \eta(\boldsymbol{w}^T\boldsymbol{x} + b),
\end{equation}
where $\eta(s) = 1/(1 + e^{-s})$, $s = \boldsymbol{w}^T\boldsymbol{x} + b$, and $\boldsymbol{x}$, $\boldsymbol{w}$ and $b$ are the input vector, weight vector and bias, respectively. Two properties of $\eta(s)$ are required for further discussions:
\begin{itemize}
\item[1.] \textbf{Quasi-zero}. The activation function satisfies $\lim_{s \to -\infty} \eta(s) = 0$.
\item[2.] \textbf{Monotonicity}. If $s_1 > s_2$, we have $\eta(s_1) > \eta(s_2)$, and so is the scaled version $\eta(as)$ with $a > 0$.
\end{itemize}

Nearly all the definitions and conclusions of this section could be generalized to the sigmoid-unit case, due to the two reasons below. The first is that the quasi-zero property of a sigmoid unit, analogous to the zero-output property of a ReLU, could lead to the generalization of definition 13 in the sense of approximation meaning, because this property could result in distinct piecewise components of a function.

The second reason is for the generalization of the remaining concepts and results except for definition 13. By equation 5.4, the output of a sigmoid unit is based on the original output $s = \boldsymbol{w}^T\boldsymbol{x} + b$ of hyperplane $\boldsymbol{w}^T\boldsymbol{x} + b = 0$. Given two points $\boldsymbol{x}_1$ and $\boldsymbol{x}_2$, let $s_1 = \boldsymbol{w}^T\boldsymbol{x}_1 + b$ and $s_2 = \boldsymbol{w}^T\boldsymbol{x}_2 + b$. If $s_1 = s_2$, the outputs of a sigmoid unit satisfy $\eta(s_1) = \eta(s_2)$, and by the monotonicity property of the activation function, if $s_1 \ne s_2$, then $\eta(s_1) \ne \eta(s_2)$, which is the rationale of the related generalizations.

Note that the second reason applies to both the tanh-unit and linear-unit cases. Thus, the model of section 5 except for definition 13 can be generalized to the tanh-unit and linear-unit networks.

\section{Explanation of Experiments}
The next two sections explain some experimental results of autoencoders based on the theories of the preceding sections. The following theorem is a preliminary to further investigations, which is trivial by the results of sections 5.
\begin{thm}
To an autoencoder $\mathcal{A}$ of equation 2.1, suppose that $\boldsymbol{x}_0$ is a point of the $m$-dimensional input space and $\boldsymbol{x}$ is obtained by perturbing $\boldsymbol{x}_0$. If the perturbation becomes a minor feature in some layer of the encoder of $\mathcal{A}$, then $\mathcal{A}$ could recover $\boldsymbol{x}_0$ from the perturbed input $\boldsymbol{x}$.
\end{thm}
\begin{proof}
This conclusion is a direct consequence of theorem 12 or its generalized version of the non-ReLU case.
\end{proof}

\subsection{Linear-Unit Autoencoder}
\citet*{Cottrell1987} pointed out that linear-unit autoencoders can do a perfect job in image compression and reconstruction, even comparable with nonlinear-unit autoencoders. This phenomenon was somewhat elusive or mysterious, since it is usually thought that a linear network is unqualified for complex tasks. We give an explanation here.

\begin{prp}
Autoencoder $\mathcal{A}$ of equation 2.1 composed of linear units can reconstruct a single point $\boldsymbol{x}_0$ of the $m$-dimensional input space; and the generalization mechanism of theorem 13 applies to $\mathcal{A}$.
\end{prp}
\begin{proof}
The compression and reconstruction of a single point $\boldsymbol{x}_0$ is trivial. Whatever the encoder is designed, we can always reconstruct $\boldsymbol{x}_0$ via the decoder constructed by theorem 11 of \citet*{Huang2022}.

To the generalization mechanism, sections 5.3 indicates that theorem 13 is applicable to linear-unit networks. In \citet*{Cottrell1987}, the point $\boldsymbol{x}_0$ is an image and a linear-unit autoencoder $\mathcal{A}_0$ was trained to reconstruct $\boldsymbol{x}_0$; when noise was added into the image $\boldsymbol{x}_0$, $\mathcal{A}_0$ could still recover $\boldsymbol{x}_0$ to some extent, and the reason can be explained by theorem 13.
\end{proof}

\subsection{Denoising Autoencoder}
This section is based on the original paper \citet*{Vincent2010} of denoising Autoencoders. We want to find the rationale behind the phenomenon of \citet*{Vincent2010}. A denoising autoencoder doesn't change the architecture of equation 2.1, but only corrupt the input $\boldsymbol{x}_0$ to be $\boldsymbol{x}'_0$ by noise. Although simple, this operation improves the performance of autoencoders \citep*{Vincent2010}.

\begin{thm}
To reconstruct $\boldsymbol{x}_0$ from a corrupted $\boldsymbol{x}'_0$ via an autoencoder $\mathcal{A}$ of equation 2.1, the encoder $\mathcal{E}$ could be designed in such a way that each layer tries to make the perturbation from noise a minor feature. In this process of dimensionality reduction, the dimensions that have physical meanings are more likely to be preserved, and the dimensions related to noise are forced to be removed in terms of minor features, through which a better lower-dimensional representation of $\boldsymbol{x}_0$ by $\mathcal{E}$ could be obtained.
\end{thm}
\begin{proof}
According to theorem 13, if the noise becomes a minor feature by the construction or training of $\mathcal{E}$, it would be removed by $\mathcal{E}$. The dimensions associated with the variation of noise will be neglected in terms of minor features, which may compensate for some dimensions caused by the physical change of $\boldsymbol{x}_0$.

When the input is a single point $\boldsymbol{x}_0$, we cannot use statistical methods (such as PCA) to extract its principal dimensions, and which dimension accounts for useful information is unknown. If nothing is done, the dimensionality reduction of $\mathcal{E}$ is not restricted to preserving the dimensions that have physical meanings. Thus, it is possible for $\mathcal{E}$ that the useful dimensions are dismissed, and the noise dimensions are preserved, resulting in a worse lower-dimensional representation of $\boldsymbol{x}_0$. The introduction of noise to $\boldsymbol{x}_0$ is helpful to avoid this drawback.
\end{proof}

\section{Variational Autoencoder}
This section explains the mechanism of variational autoencoders originated from \citet*{Kingma2014}. We assume that the type of unit is sigmoid or ReLU; the results not related to definition 13 can be applicable to the tanh-unit case.

Throughout the discussion, it is beneficial to bear in mind an intuitive example of the input point $\boldsymbol{x}$, that is, an image (zigzag-order version) to be represented or reconstructed by an autoencoder, such as in \citet*{Kingma2014}.

\subsection{Basic Model}
A variational autoencoder was proposed by \citet*{Kingma2014}, whose encoder can be expressed as a three-layer network
\begin{equation}
\boldsymbol{h} = \tau(\boldsymbol{w}_3\boldsymbol{x} + \boldsymbol{b}_3),
\end{equation}
\begin{equation}
\boldsymbol{\mu}_e = \boldsymbol{w}_4\boldsymbol{h} + \boldsymbol{b}_4, \boldsymbol{\sigma}_e^2 = \boldsymbol{w}_5\boldsymbol{h} + \boldsymbol{b}_5
\end{equation}
coupled with the probabilistic model of the output
\begin{equation}
\boldsymbol{x}_e \sim \mathcal{N}(\boldsymbol{\mu}_e, \boldsymbol{\Sigma}_{\boldsymbol{\sigma}_e^2}),
\end{equation}
where $\tau$ is the activation function of a unit, $\boldsymbol{w}_i$ and $\boldsymbol{b}_i$ for $i = 3, 4, 5$ are the weight and bias parameters of the encoder, $\boldsymbol{x}_e$ is the output of the encoder with $\mathcal{N}$ representing Gaussian distribution, and $\boldsymbol{\Sigma}_{\boldsymbol{\sigma}_e^2}$ is a diagonal matrix whose diagonal entries correspond to the entries of $\boldsymbol{\sigma}_e^2$ of equation 7.2, respectively (similarly for $\boldsymbol{\Sigma}_{\boldsymbol{\sigma}^2}$ below).

In Equations 8.1, $\boldsymbol{x}$ is a vector of the input space, and $\boldsymbol{h}$ a vector of the space of the hidden layer. Equation 7.2 of the output layer gives $\boldsymbol{\mu}$ and $\boldsymbol{\sigma}_e^2$, which are the mean vector and covariance vector (whose components are the diagonal entries of the covariance matrix) of the distribution of equation 7.3. The sampling of equation 7.3 yields the encoding output $\boldsymbol{x}_e$.

Note that in \citet*{Kingma2014}, the left side of the second term of equation 7.2 is $\log \boldsymbol{\sigma}_e^2$. The $\log$ operation may be useful to the learning; however, from the perspective of reconstruction capabilities, the $\log$ operation is unnecessary.

The decoder is similar to the encoder: a three-layer network
\begin{equation}
\boldsymbol{h}' = \tau(\boldsymbol{w}'_3\boldsymbol{x}_e + \boldsymbol{b}'_3),
\end{equation}
\begin{equation}
\boldsymbol{\mu} = \boldsymbol{w}'_4\boldsymbol{h}' + \boldsymbol{b}'_4, {\boldsymbol{\sigma}}^2 = \boldsymbol{w}'_5\boldsymbol{h}' + \boldsymbol{b}'_5,
\end{equation}
together with the probabilistic model
\begin{equation}
\boldsymbol{x} \sim \mathcal{N}(\boldsymbol{\mu}, \boldsymbol{\Sigma}_{\boldsymbol{\sigma}^2}),
\end{equation}
whose sampling is the final output. For an intuitive illustration of this kind of autoencoders, the reader is referred to Figure 2.2 of \citet*{Girin2020}.

\subsection{Loss Function}
In order to highlight the main factor, we'll use the quadratic loss function to describe the mechanism of variational autoencoders. Note that the loss function (estimator) of \citet*{Kingma2014} was designed on the basis of probabilistic models and one of its two terms is a maximum likelihood estimation, which is closely related to quadratic loss functions when the distribution is Gaussian.

In fact, whatever the type of the loss function is, the ultimate goal of an autoencoder is to reconstruct the input data, and the reconstruction effect would be finally reflected in the quadratic loss function.

Furthermore, to some other loss functions, our theoretical framework could be modified in relevant assumptions and proofs, such that the main conclusions are still applicable.

\subsection{Encoder-Output Disturbance}
For convenience, all the explanations are in the form of the single-point reconstruction, and the generalization to the multi-point case is trivial by the mechanism of multi-output neural networks \citep*{Huang2022}.

The following assumption is not a general conclusion that holds for all the cases, but would be true in some certain scenarios.
\begin{assm}
To reconstruct a point $\boldsymbol{x}$ of the input space by an autoencoder of equation 2.1, we assume that the local generalization leads to smaller reconstruction error than the nonlocal generalization.
\end{assm}

The lemma below shows how the random disturbance of the encoding vector influences the parameter setting of the decoder.
\begin{lem}
To the decoder $\mathcal{D}$ of a variational autoencoder $\mathcal{A}$, under assumption 1, regardless of equation 7.6, the random disturbance of the encoding vector by equation 7.3 could enforce local generalizations.
\end{lem}
\begin{proof}
Equation 7.3 is equivalent to
\begin{equation}
\boldsymbol{x}_{e} = \boldsymbol{\mu}_e + \boldsymbol{\varepsilon},
\end{equation}
where $\boldsymbol{\varepsilon} \sim \mathcal{N}(\boldsymbol{0}, \boldsymbol{\Sigma}_{\boldsymbol{\sigma}_e^2})$, and $\boldsymbol{x}_{e}$ could be regarded as a random variation from $\boldsymbol{\mu}_e$. Let $H$ be the set of the hyperplanes derived from the units of the hidden layer of $\mathcal{D}$. By corollary 12 of \citet*{Huang2022}, the output of $\mathcal{D}$ is determined by $H$ along with its output weights.

Vector $\boldsymbol{x}_e$ randomly changes in a region of the encoding space, whose size is controlled by the covariance matrix $\boldsymbol{\Sigma}_{\boldsymbol{\sigma}_e^2}$. The training of $\mathcal{D}$ according to the loss function would force the reconstruction error with respect to the random input $\boldsymbol{x}_e$ to be as small as possible. Under assumption 1, the variation of $\boldsymbol{x}_{e}$ from $ \boldsymbol{\mu}_e$ should be within a same divided region of $H$, since this could result in local generalizations and thus the reconstruction error is smaller.

For fixed $\boldsymbol{\mu}_e$ and $\boldsymbol{\Sigma}_{\boldsymbol{\sigma}_e^2}$, the updating of the parameters of $H$ due to the training could reformulate the divided regions to make the output of the variation of $\boldsymbol{x}_{e}$ a local generalization. Thus, adding noise by equation 7.3 is beneficial to the optimization in terms of local generalizations.
\end{proof}

The next lemma is for the preparation of assumption 2 later.
\begin{lem}
To a point $\boldsymbol{x}$, suppose that an autoencoder $\mathcal{A}$ of equation 2.1 has been trained to reconstruct it, and that the variation of $\boldsymbol{x} + \boldsymbol{\varepsilon}_1$ becomes a minor feature in some layer of the encoder, whereas another variation of $\boldsymbol{x} + \boldsymbol{\varepsilon}_2$ does not. Then $\boldsymbol{x} + \boldsymbol{\varepsilon}_1$ has no possibility to be reconstructed by $\mathcal{A}$, but $\boldsymbol{x} + \boldsymbol{\varepsilon}_2$ may be possible.
\end{lem}
\begin{proof}
Compared to the original input $\boldsymbol{x}$, the perturbed $\boldsymbol{x} + \boldsymbol{\varepsilon}_1$ yields an overlapping map in the encoding space, due to the property of minor features; thus, the output of $\mathcal{A}$ with respect to $\boldsymbol{x} + \boldsymbol{\varepsilon}_1$ is $\boldsymbol{x}$, not $\boldsymbol{x} + \boldsymbol{\varepsilon}_1$. Let $\boldsymbol{x}_e$ be the encoding vector of input $\boldsymbol{x}$. Because  point $\boldsymbol{x} + \boldsymbol{\varepsilon}_2$ could result in a variation of $\boldsymbol{x}_e$ in the encoding space, it is possible that $\mathcal{A}$ outputs $\boldsymbol{x} + \boldsymbol{\varepsilon}_2$.
\end{proof}

The assumption that follows is natural for the usefulness of autoencoders, but difficult to be proved at present. The previous lemma 13 provides some reasonableness for it in terms of rigorous proofs.
\begin{assm}
Suppose that an autoencoder $\mathcal{A}$ of equation 2.1 has been trained or constructed to reconstruct an input $\boldsymbol{x}$. If the variation of the input $\boldsymbol{x} + \boldsymbol{\varepsilon}$ doesn't become a minor feature in any layer of the encoder, we assume that the output of $\mathcal{A}$ with respect to $\boldsymbol{x} + \boldsymbol{\varepsilon}$ could be $\boldsymbol{x} + \boldsymbol{\varepsilon}$.
\end{assm}

The following assumption is established due to the same reason as assumption 1.
\begin{assm}
To an autoencoder of equation 2.1, in comparison with an input $\boldsymbol{x}$ of the input space, under some metric, we assume that the error $\|\boldsymbol{\varepsilon}\| = \|(\boldsymbol{x} + \boldsymbol{\varepsilon}) - \boldsymbol{x}\|$ of corrupted $\boldsymbol{x} + \boldsymbol{\varepsilon}$ by noise is larger than $\|\Delta \boldsymbol{x}\|$ of a smoothly changed $\boldsymbol{x} + \Delta \boldsymbol{x}$, whose variation is smaller due to physical background reasons.
\end{assm}

\noindent
\textbf{Example.} Assume that the input point $\boldsymbol{x}$ is an image of a boy's face. If the boy smiles, the image also changes; this is a minor smooth variation that has physical meaning. And if the image is corrupted by salt-and-pepper noise globally, the caused error could be larger than that of the smiled case.

\begin{lem}
Under assumptions 2 and 3, regardless of equation 7.6, the random disturbance of the encoding vector of equation 7.3 could make the dimensions of the encoding space correspond to smooth variations of the input $\boldsymbol{x}$, rather than noise variations. That is, the dimensions for smooth variations tend to be preserved and the dimensions for noise variations are more likely to be dismissed.
\end{lem}
\begin{proof}
During dimensionality reduction of the encoder, assume that the dimensions (collectively denoted by set $N$) for noise $\boldsymbol{\varepsilon}$ are preserved, while those (denoted by set $S$) for smooth variation $\Delta \boldsymbol{x}$ are neglected. Then the dimensions of the encoding space correspond to noise variation $\boldsymbol{\varepsilon}'$ of the input space.

The disturbed encoding vector $\boldsymbol{x}_e$ of equation 7.3 has its associated output and input of the autoencoder. Without using random disturbance, to produce the varied $\boldsymbol{x}_e$ by the encoder, the input point should be in the form of $\boldsymbol{x} + \boldsymbol{\varepsilon}'$. By assumption 2, correspondingly, the output of the decoder is also the noisy version $\boldsymbol{x} + \boldsymbol{\varepsilon}'$. But in this case, by assumption 3, the reconstruction error would be larger and the training process would tend to change this situation to render the error smaller.

When the encoder preserves the dimensions of $S$, by the similar analysis as above, the decoder would output the smooth version $\boldsymbol{x} + \Delta \boldsymbol{x}$, whose reconstruction error is smaller. This is the state that the training process more likely reaches, due to the optimization for the loss function.
\end{proof}

\begin{rmk}
This effect is similar to that of a denoising autoencoder of theorem 14.
\end{rmk}

\begin{thm}
Under assumptions 1, 2, and 3, regardless of equation 7.6, the random disturbance of the encoding vector of equation 7.3 could influence the training of a variational autoencoder in two ways. One is for the decoder to enforce local generalizations. The other is to make the encoder tend to preserve the smooth-variation dimensions, rather than the noise-variation ones.
\end{thm}
\begin{proof}
This theorem is the combination of lemmas 12 and 14.
\end{proof}

\subsection{Decoder-Output Disturbance}
\begin{lem}
Under assumption 1, regardless of equation 7.3, the training of a variational autoencoder tends to make the variation of the output of its decoder $\mathcal{D}$, which is due to the disturbance of equation 7.6, to form a local generalization.
\end{lem}
\begin{proof}
The proof is similar to that of lemma 12. In this case, the output of $\mathcal{D}$ is directly perturbed, instead of being caused by the disturbance of encoding vectors. To a variation of the output of $\mathcal{D}$, if the corresponding point of its input space is within a divided region of the hyperplanes of the hidden layer, it is a local generalization and the associated reconstruction error would be smaller, according to assumption 1.
\end{proof}

\begin{lem}
Under assumptions 2 and 3, regardless of equation 7.3, the disturbance of the decoder output of equation 7.6 could make the variations of the reconstructed data correspond to the smooth changes of the input $\boldsymbol{x}$.
\end{lem}
\begin{proof}
If the disturbance of equation 7.6 results in noisy reconstruction, by assumptions 2 and 3, the loss function would have a larger value and the training would force the variations changed to be smooth ones.
\end{proof}

\begin{thm}
Under assumptions 1, 2, and 3, regardless of equation 7.3, there are two ways for the disturbance of equation 7.6 to influence the training of a variational autoencoder. The first is to improve the optimization via local generalizations. The second is to output smooth-variation reconstructed data.
\end{thm}
\begin{proof}
This is the summary of lemmas 15 and 16.
\end{proof}

\subsection{Mechanism of Image Restoration}
During the investigation of denosing autoencoders and variational encoders, the mechanism of image restoration via autoencoders is in fact revealed.

\begin{prp}
Suppose that an autoencoder $\mathcal{A}$ of equation 2.1 has been trained to reconstruct an image $\mathcal{I}$. If the input to $\mathcal{A}$ is a degraded version $\mathcal{I}'$ and the original $\mathcal{I}$ can be restored by $\mathcal{A}$, one reason is that the degradation is neglected by the encoder in terms of minor features.
\end{prp}
\begin{proof}
This conclusion is a special case of theorem 13.
\end{proof}

\begin{rmk}
There are many other image restoration methods, including deconvolution \citep*{Banham1997}, iteration \citep*{Biemond1990}, partial differential equations \citep*{Rudin1992}, and so on \citep*{Gonzalez(2008)}. The autoencoder way is task dependent and not applicable to general images. However, the main application of autoencoders is not image restoration, but the lower-dimensional representation of an image for further uses \citep*{Bengio2006}.
\end{rmk}

\subsection{Summary}
Since it has been known for a long time that a three-layer network composed of sigmoid units is a universal approximator {\citep*{Hornik1989,Cybenko1989,Hecht-Nielsen1989}}, it's natural to use it to construct an autoencoder, which had been done by \citet*{Kramer1991}.

However, it's nontrivial that the introduction of a probabilistic model could result in the benefits as the discussion of variational autoencoders and denoising autoencoders above, despite the preceding work of \citet*{Bishop1995}. A comment on variational autoencoders by \citet*{Goodfellow2016} is ``elegant, theoretically pleasing, and simple to implement'' .

For experiments of variational autoencoders, Figure 4 of \citet*{Kingma2014} is a good example. In Figure 4a, the two-dimensional encoding space grasped two smooth variations of a face, corresponding to the rotation and emotion, respectively, which have clear physical meaning. Figure 4b provided another example, which also obtained a meaningful lower-dimensional representation of some handwritten characters. They both demonstrated the excellent performance of variational autoencoders in lower-dimensional representation of data.

\section{Discussion}
The solutions of an autoencoder constructed by this paper may only be some special cases of the solution space; however, the underlying principles may be the rules that govern applications, such as the properties of bijective maps and data disentangling of the encoder. More solutions should be explored under this framework, until the one that the training method reaches is found or verified.

The most interesting or useful property of autoencoders is the lower-dimensional representation of higher-dimensional data, which needs further study on more specific problems. Our explanation of denoising autoencoders and variational autoencoders is only a general framework.

A convolutional neural network could be interpreted from the perspective of autoencoders to some extent, which may enrich our knowledge on this type of architectures. The comparison to decision trees and PCA suggests the advantage of deep learning over some other machine learning methods. We hope that this paper would pave the way for the understanding of neural networks associated with autoencoders.

\bibliographystyle{APA}

\end{document}